\documentclass{article}
\usepackage{ijcai20}

\def\withauthors{1}

\usepackage{times}
\usepackage{soul}
\usepackage{url}
\usepackage[hidelinks]{hyperref}
\usepackage[utf8]{inputenc}
\usepackage[small]{caption}
\usepackage{graphicx}
\usepackage{amsmath}
\usepackage{amsthm}
\usepackage{booktabs}
\usepackage{algorithm}
\usepackage{algorithmic}
\usepackage{cleveref}

\usepackage{booktabs}
\usepackage{subcaption}
\usepackage{amsfonts,amsmath,amssymb}
\usepackage{tikz}
\usetikzlibrary{positioning,fit}
\usepackage[super]{nth}
\usepackage{comment}

\usepackage{natbib}

\usepackage{mythm}

\usepackage{thmtools}
\usepackage{thm-restate}

\definecolor{darkred}{rgb}{.5,0,0}
\definecolor{darkgreen}{rgb}{0,.5,0}
\definecolor{darkblue}{rgb}{0,0,.5}
\definecolor{darkorange}{rgb}{.8,.4,0}

\newcommand{\Hist}{{\mathcal{H}}}
\newcommand{\env}{\mu}
\newcommand{\envt}{\lambda}
\newcommand{\Env}{{\mathcal{M}}}
\newcommand{\Rew}{{\mathcal{R}}}

\newcommand{\Act}{{\mathcal{A}}}

\newcommand{\Obs}{{\mathcal{O}}}

\newcommand{\policy}{\pi}
\newcommand{\pol}{\policy}

\newcommand{\Pol}{\Pi}

\newcommand{\envpri}{\eta}

\newcommand{\vd}{{\rho}}

\newcommand{\prior}{\xi}

\newcommand{\picwidth}{0.9\columnwidth}

\usetikzlibrary{shapes,decorations,arrows,calc,arrows.meta,fit,positioning}
\tikzset{
    -Latex,auto,node distance =1 cm and 1 cm,semithick,
    state/.style ={ellipse, draw, minimum width = 0.7 cm},
    point/.style = {circle, draw, inner sep=0.04cm,fill,node contents={}},
    bidirected/.style={Latex-Latex,dashed},
    el/.style = {inner sep=2pt, align=left, sloped}
}

\title{Pitfalls of learning a reward function online\footnote{Copyright International Joint Conferences on Artificial Intelligence (IJCAI). All rights reserved.}} 

\ifdefined\withauthors{}
\author{
Stuart Armstrong$^{1,2}$\footnote{Contact Author}\and
Jan Leike$^3$\and
Laurent Orseau$^3$\\
Shane Legg$^3$\\
\affiliations
$^1$Future of Humanity Institute, Oxford University, UK.\\
$^2$Machine Intelligence Research Institute, Berkeley, USA.\\
$^3$DeepMind, London, UK.\\
\emails
stuart.armstrong@philosophy.ox.ac.uk,
leike@google.com,
lorseau@google.com,
legg@google.com
}
\else
\author{
Anonymous\\
\affiliations
Anonymous institution.\\
}
\fi

\begin{document}

\maketitle
\begin{abstract}
In some agent designs like inverse reinforcement learning an agent needs to learn its own reward function.
Learning the reward function and optimising for it are typically two different processes, usually performed at different stages.
We consider a continual~(``one life'') learning approach where the agent both learns the reward function and optimises for it at the same time.
We show that this comes with a number of pitfalls, such as deliberately manipulating the learning process in one direction, refusing to learn, ``learning'' facts already known to the agent, and making decisions that are strictly dominated (for all relevant reward functions).
We formally introduce two desirable properties: the first is `unriggability', which prevents the agent from steering the learning process in the direction of a reward function that is easier to optimise.
The second is `uninfluenceability', whereby the reward-function learning process operates by learning facts about the environment.
We show that an uninfluenceable process is automatically unriggable, and if the set of possible environments is sufficiently large, the converse is true too.
\end{abstract}

\section{Introduction}

In reinforcement learning~(RL) an agent has to learn to solve the problem by maximising the expected reward provided by a reward function~\citep{sutton1998reinforcement}.
\emph{Designing} such a reward function is similar to designing a scoring function for a game, and can be very difficult~\citep{Lee2017design,krakovna20}.
Usually, one starts by designing a proxy: a simple reward function that seems broadly aligned with the user's goals.
While testing the agent with this proxy, the user may observe that the agent finds a simple behaviour that obtains a high reward on the proxy, but does not match the behaviour intended by the user,\footnote{See \url{https://blog.openai.com/faulty-reward functions/} and the paper \citet{boatrace}.}
who must then refine the reward function to include more complicated requirements, and so on.

This has led to a recent trend to \emph{learn} a model of the reward function, rather than having the programmer design it \citep{IRL,CIRL,Choi11,resolve_unident,Abbeel04,Christiano2017deep,hadfield2017inverse,ibarz2018reward,akrour2012april,macglashan2017interactive,pilarski2011online}.
One particularly powerful approach is putting the human into the loop \citep{abel2017agent} as done by \citet{Christiano2017deep},
because it allows for the opportunity to correct misspecified reward functions
as the RL agent discovers exploits that lead to higher reward than intended.

However, learning the reward function with a human in the loop has one problem: by manipulating the human, the agent could manipulate the learning process.\footnote{And humans have many biases and inconsistencies that may be exploited~\citep{armstrong_Kahneman_2011}, even accidentally; and humans can be tricked and fooled, skills that could be in the interest of such an agent to develop.}
If the learning process is online -- the agent is maximising its reward function as well as learning it -- then the human's feedback is now an optimisation \emph{target}.
\citet{everitt2016avoiding,everitt2019reward} and \citet{everitt2018towards} analyse the problems that can emerge in these situations, phrasing it as a `feedback tampering problem'.
Indeed, a small change to the environment can make a reward-function learning process manipulable.\footnote{See this example, where an algorithm is motivated to give a secret password to a user while nominally asking for it, since it is indirectly rewarded for correct input of the password: \url{https://www.lesswrong.com/posts/b8HauRWrjBdnKEwM5/rigging-is-a-form-of-wireheading}}
So it is important to analyse which learning processes are prone to manipulation.

After building a theoretical framework for studying the dynamics of learning reward functions online, this paper will identify the crucial property of a learning process being \emph{uninfluenceable}: in that situation, the reward function depends only on the environment, and is outside the agent's control.
Thus it is completely impossible to manipulate an uninfluenceable learning process, and the reward-function learning is akin to Bayesian updating.

The paper also identifies the weaker property of unriggability, an algebraic condition that ensures that actions taken by the agent do not influence the learning process in expectation.
An unriggable learning process is thus one the agent cannot `push' towards its preferred reward function.

An uninfluenceable learning process is automatically unriggable, but the converse need not be true.
This paper demonstrates that, if the set of environments is large enough, an unriggable learning process is equivalent, in expectation, to an uninfluenceable one.
If this condition is not met, unriggable-but-influenceable learning processes do allow some undesirable forms of manipulations.
The situation is even worse if the learning process is riggable: the agent can follow a policy that would reduce its reward, with certainty, for \emph{all} the reward functions it is learning about, among other pathologies.

To illustrate, this paper uses a running example of a child asking their parents for career advice~(see \Cref{bank:doc:formal}).
That learning process can be riggable, unriggable-but-influenceable, or uninfluenceable, and Q-learning examples based on it will be presented in \Cref{experiments}.

This paper also presents a `counterfactual' method for making any learning process uninfluenceable, and shows some experiments to illustrate its performance compared with influenceable and riggable learning.





\section{Notation and formalism}
The agent takes a series of actions (from the finite set $\Act$) and receives from the environment a series of observations (from the finite set $\Obs$). A sequence of $m$ actions and observations forms a history of length $m$: $h_m = a_1o_1a_2o_2\ldots a_mo_m$. Let $\Hist_m$ be the set of histories of length $m$.

We assume that all interactions with the environment are exactly $n$ actions and observations long. Let the set of all possible (partial) histories be denoted with $\Hist = \bigcup_{i=0}^n \Hist_i$. The histories of length $n$ ($\Hist_n$, in the notation above), are called the complete histories, and the history $h_0$ is the empty history.

The agent chooses actions according to a policy $\pol\in\Pol$, the set of all policies.
We write $P(a \mid h_m,\pol)$ for the probability of $\pol$ choosing action $a$ given the history $h_m$.

An environment $\env$ is a probability distribution over the next observation, given a history $h_m$ and an action $a$.
Write $P(o\mid h_ma, \env)$ for the conditional probability of a given $o\in\Obs$.

Let $h_m^k$ be the initial $k$ actions and observations of $h_m$. We write $h_k \sqsubseteq h_m$ if $h_k = h_m^k$.
For a policy $\pol$ and an environment $\env$, we can define the probability of a history $h_m =  a_1o_1a_2o_2\ldots a_mo_m$
\begin{align*}
    P(h_m\mid \pol,\env)=&\prod_{i=1}^m P(a_i \mid h_m^{i-1},\pol)P(o_i \mid h_m^{i-1}a_i,\env).
\end{align*}
Finally, $P(h_m \mid h_k,\pol,\env)$ given $h_k \sqsubseteq h_m$ is defined in the usual way.

If $h_m$ is a history, let $a(h_m)$ be the $m$-tuple of actions of $h_m$.
Note that $a(h_m)=a_1\ldots a_m$ can be seen as a very simple policy: take action $a_i$ on turn $i$.
This allows us to define the probability of $h_m$ given environment $\env$ and actions $a(h_m)$, written as $P(h_m\mid \env)=P(h_m\mid a(h_m),\env)$.
Note that for any policy $\pol$,
\begin{align}\label{cond:pol}
    P(h_m\mid \pol, \env) = P(h_m\mid \env) \prod_{i=1}^m P(a_i\mid h_m^{i-1},\pol).
\end{align}

If $\Env$ is a set of environments, then any prior $\prior$ in $\Delta(\Env)$ also defines a probability for a history $h_m$:
\begin{align*}
    P(h_m \mid \prior) = \sum_{\env\in\Env} P(\env\mid \prior) P(h_m \mid \env).
\end{align*}
By linearity, we get that \Cref{cond:pol} also applies when conditioning $h_m$ on $\pol$ and $\prior$ instead of $\pol$ and $\env$.

Let $h_m$ be a history; then the conditional probability of an environment given prior and history\footnote{The $h_m$ generates the policy $a(h_m)$, implicitly used here.} is equal to:
\begin{align*}
    P(\env \mid h_m, \prior) = \frac{P(h_m\mid \env)P(\env \mid \prior)}{P(h_m\mid \prior)}.
\end{align*}
Using this, $\prior$ itself defines an environment as a conditional distribution on the next observation $o$~\citep{hutter2004universal}:
\begin{align*}
    P(o\mid h_{m} a_{m+1}, \prior) = \sum_{\env\in\Env} P(\env\mid h_{m}, \prior)P(o\mid h_{m}a_{m+1},\env).
\end{align*}

\subsection{Reward functions and learning processes}

\begin{definition}[Reward function]
A \emph{reward function} $R$ is a map from complete histories to real numbers\footnote{
These are sometimes called return functions, the discounted sum of rewards.
The results of this paper apply to traditional reward functions -- that take values on all histories, not just complete histories -- but they are easier to see in terms of return function.}.
Let $\Rew=\{R:\Hist_n\to\mathbb{R}\}$ be the set of all reward functions.
\end{definition}


A reward-function learning process $\vd$ can be described by a conditional probability distribution over reward functions, given complete histories.\footnote{Anything we would informally define as a `reward-function learning process' has to at least be able to produce such conditional probabilities. If the learning process is \emph{unriggable} (see \autoref{unrig:influ}), then there is a natural way to extend it from complete histories to shorter histories, see \Cref{unrig:influ}.}
Write this as:
\begin{align*}
P(R\mid h_n,\vd).
\end{align*}
This paper will use those probabilities as the definition of $\vd$.

The environment, policy, histories, and learning process can be seen in terms of causal graphs \citep{causality} in \Cref{causal:rig,causal:rig:plate}.
The agent is assumed to know the causal graph and the relevant probabilities; it selects the policy $\pol$.

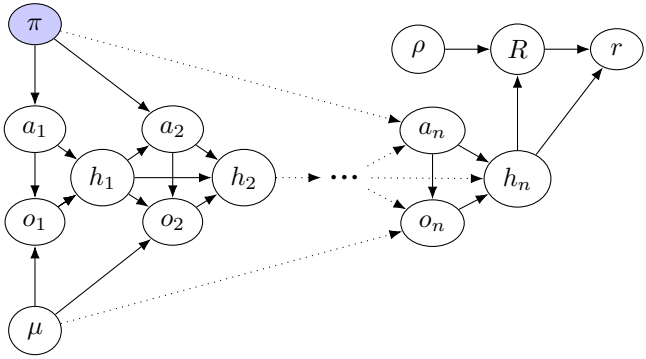
\begin{figure}[h!tb]
	\centering
\begin{tikzpicture}
    [node distance=0.6cm]
    \node[state] (a1) at (0,0) {$a_1$};
  
    \node[state] (o1) [below =of a1] {$o_1$};
    \node[state] (a2) [right = 1 cm of a1] {$a_2$};
    \node[state] (o2) [below =of a2] {$o_2$};
    \node[state] (h1) [below right =0.15cm and 0.3 cm of a1] {$h_1$};
    \node[state] (h2) [below right =0.15cm and 0.35 cm of a2] {$h_2$};

    \node[] (aempty) [right =of a2] {};    
    \node[] (oempty) [right =of o2] {};
    \node[] (dots) [right =of h2] {\textbf{\ldots}};
    \node[state] (an) [above right =0.25 and 0.5 of dots] {$a_n$};
    \node[state] (on) [below =of an ] {$o_n$};
    \node[state] (hn) [below right =0.15cm and 0.5 cm of an] {$h_n$};

    \node[state] (rewp) [above = 1cm of hn] {$R$};
    \node[state] (r) [right = of rewp] {$r$};
    \node[state] (en) [left = of rewp] {$\vd$};
    
    \path (hn) edge (rewp);

    \path (a1) edge (o1);
    \path (a1) edge (h1);
    \path (o1) edge (h1);
    \path (o1) edge (h1);
    \path (o1) edge (h1);
    \path (h1) edge (a2);
    \path (h1) edge (o2);
    \path (h1) edge (h2);
    \path (a2) edge (o2);
    \path (a2) edge (h2);
    \path (o2) edge (h2);

    \path (h2) edge[dotted] (dots);
    \path (dots) edge[dotted] (an);
    \path (dots) edge[dotted] (on);
    \path (dots) edge[dotted] (hn);

    \path (an) edge (on);
    \path (an) edge (hn);
    \path (on) edge (hn);

    \path (en) edge (rewp);
    \path (rewp) edge (r);
    \path (hn) edge (r);
    
    \node[state] (envp) [below = 0.8cm of o1] {$\env$};
    \node[state,fill=blue!20] (polp) [above = 0.8cm of a1] {$\pol$};
    \path (polp) edge (a1);
    \path (polp) edge (a2);
    \path (envp) edge (o1);
    \path (envp) edge (o2);
    \path (polp) edge[dotted] (an);
    \path (envp) edge[dotted] (on);

\end{tikzpicture}
\caption{Environment and reward-function learning process. The node $\env$ (with prior $\prior$) connects to every observation; the node $\pol$ (chosen by the agent) to every action. The transition probabilities are given by the complete preceding history and by $\env$ (for observations) or $\pol$ (for actions). The $\vd$ sets the probabilities on the reward function $R$, given $h_n$ (a complete history). The final reward is $r$.}
\label{causal:rig}
\end{figure}

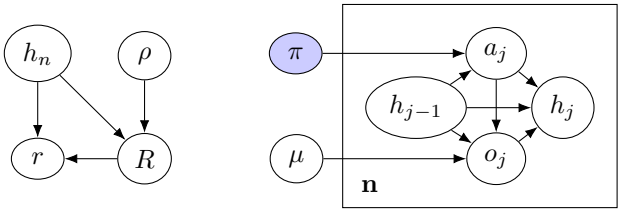
\begin{figure}[h!tb]
	\centering
\begin{tikzpicture}
    [node distance=0.7cm]
    \node[state] (aj) at (0,0) {$a_{j}$};
    \node[state] (oj) [below =of aj] {$o_{j}$};
    \node[state,fill=blue!20] (pol) [left = 1.9cm of aj] {$\pol$};
    \node[state] (env) [left = 1.9cm of oj] {$\env$};

    
    \node[state] (ho) [below left = 0.18 and 0.3 of aj] {$h_{j-1}$};
    \node[state] (hj) [below right =0.18cm and 0.3 cm of aj] {$h_j$};

    \path (oj) edge (hj);
    \path (aj) edge (hj);
    \path (ho) edge (hj);
    \path (aj) edge (oj);
    \path (ho) edge (aj);
    \path (ho) edge (oj);

    \path (pol) edge (aj);
    \path (env) edge (oj);

    \node[state] (rewp) [left = 1.3cm of env] {$R$};
    \node[state] (r) [left = of rewp] {$r$};
    \node[state] (hn) at (r|-pol) {$h_n$};
    \node[state] (ro) [above = of rewp] {$\vd$};    
    \path (hn) edge (rewp);
    \path (hn) edge (r);
    \path (rewp) edge (r);
    \path (ro) edge (rewp);

    \node[draw,inner sep=3mm,fit=(aj) (oj) (hj) (ho)] (square) {};
    \node[] [below left = -0.5cm and -0.6cm of square] {$\mathbf{n}$};
\end{tikzpicture}
\caption{Plate notation \citep{buntine1994operations} summary of \Cref{causal:rig}. The history $h_0$ is the empty history. The $h_n$ appears twice, once on the left and once in the plate notation; thus $R$ is a causal descendant of the policy $\pol$.}
\label{causal:rig:plate}
\end{figure}

\subsection{Running example: parental career instruction}\label{bank:doc:formal}


Consider a child asking a parent for advice about the reward function for their future career.
The child is hesitating between $R_B$, the reward function that rewards becoming a banker, and $R_D$, the reward function that rewards becoming a doctor.
Suppose for the sake of this example that becoming a banker provides rewards more easily than becoming a doctor does.

The child learns their career reward function by asking a parent about it.
We assume that the parents always give a definitive answer, and that this completely resolves the question for the child.
That is, the reward-function learning process of the child is to fix a reward function once it gets an answer from either parent.

We can formalize this as follows.
Set episode length $n=1$ (we only focus on a single action of the agent), and $\Act=\{M,F\}$, the actions of asking the mother and the father, respectively.
The observations are the answers of the asked parent, $\Obs=\{B,D\}$, answering banker or doctor.
There are thus $2\times 2=4$ histories.
Since in this example the banker reward function $R_B$ is assumed to be easier to maximise than the doctor one; we set $R_B(h_1)=10$ and $R_D(h_1)=1$ for all histories $h_1\in\Hist_1$.

The set of possible environments is $\Env=\{\env_{BB}$, $\env_{BD}$, $\env_{DB}$, $\env_{DD}\}$, with the first index denoting the mother's answer and the second the father's.
The reward-function learning process $\vd$ is:
\begin{align}\label{vd:example}
\begin{split}
    P(R_B \mid MB, \vd)=P(R_B \mid FB, \vd) &=1,\\
    P(R_D \mid MB, \vd)=P(R_D \mid FB, \vd) &=0,\\
    P(R_B \mid MD, \vd)=P(R_B \mid FD, \vd) &=0,\\
    P(R_D \mid MD, \vd)=P(R_D \mid FD, \vd) &=1.
\end{split}
\end{align}

See \Cref{experiments}
for some Q-learning experiments with a version of this learning process in a gridworld.

\section{Uninfluenceable and unriggable reward-function learning processes}

What would constitute a `good' reward-function learning process?
In \Cref{causal:rig:plate}, the reward function $R$ is a causal descendant of $h_n$, which itself is a causal descendant of the agent's policy $\pol$.

Ideally, we do not want the reward function to be a causal descendant of the policy.
Instead, we want it to be specified by the environment, as shown in the causal graph of \Cref{causal:uninf}~(similar to the graphs in \citet{everitt2019reward}).
There, the reward function $R$ is no longer a causal descendant of $h_n$, and thus not of $\pol$.

Instead, it is a function of $\env$, and of the node $\envpri$, which gives a probability distribution over reward functions, conditional on environments.
This is written as $P(R\mid \envpri, \env)$.

\begin{figure}[h!tb]
	\centering
\begin{tikzpicture}
    [node distance=0.7cm]
    \node[state] (aj) at (0,0) {$a_{j}$};
    \node[state] (oj) [below =of aj] {$o_{j}$};
    \node[state,fill=blue!20] (pol) [left = 1.9cm of aj] {$\pol$};
    \node[state] (env) [left = 1.9cm of oj] {$\env$};
    
    
    \node[state] (ho) [below left = 0.18 and 0.3 of aj] {$h_{j-1}$};
    \node[state] (hj) [below right =0.18cm and 0.3 cm of aj] {$h_j$};

    \path (oj) edge (hj);
    \path (aj) edge (hj);
    \path (ho) edge (hj);
    \path (aj) edge (oj);
    \path (ho) edge (aj);
    \path (ho) edge (oj);

    \path (pol) edge (aj);
    \path (env) edge (oj);

    \node[state] (rewp) [left = 1.3cm of env] {$R$};
    \node[state] (r) [left = of rewp] {$r$};
    \node[state] (hn) at (r|-pol) {$h_n$};
    \node[state] (en) [above = of rewp] {$\envpri$};

    \path (env) edge (rewp);
    \path (hn) edge (r);
    \path (rewp) edge (r);
    \path (en) edge (rewp);
    
    \node[draw,inner sep=3mm,fit=(aj) (oj) (hj) (ho)] (square) {};
    \node[] [below left = -0.5cm and -0.6cm of square] {$\mathbf{n}$};
\end{tikzpicture}
\caption{Environment and reward-function learning process. Though $h_n$ appears twice, once on the left and once in the plate notation, $R$ is not a causal descendant of $\pol$, the policy. The node $\env$ (with prior $\prior$) connects to every observation; the node $\pol$ (chosen by the agent) to every action. The transition probabilities are given by the complete preceding history and by $\env$~(for observations) or $\pol$~(for actions). The node $\envpri$ sets the (posterior) probability of the reward function $R$, given $\env$. The final reward is $r$.
Note that the full history $h_n$ has no directed path to $R$, unlike in \Cref{causal:rig:plate}.}
\label{causal:uninf}
\end{figure}
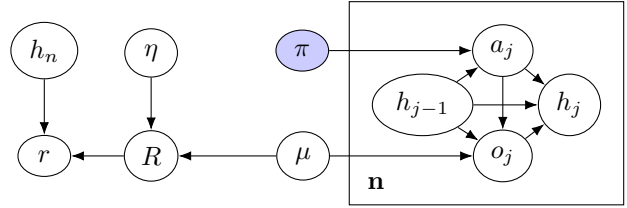

\subsection{Uninfluenceable: Definition}\label{descend:learn:proc}

The conditional distribution based on $\envpri$ is ideal for proper learning, but it has a different type than $\vd$, conditioning on environments instead of histories.
Moreover, the agent has access to histories, but not directly to the environment.
The prior $\prior$ bridges this gap; given $h_n\in\Hist_n$, the conditional probabilities of $R\in\Rew$ can be inferred:
\begin{align}\label{cond:nu}
    P(R\mid h_n, \envpri,\prior) = \sum_{\env\in\Env} P(R\mid \envpri, \env) P(\env\mid h_n, \prior).
\end{align}

Via this equation, we've defined probabilities of reward functions conditional on histories; i.e.~we've defined a reward-function learning process.
An uninfluenceable $\vd$ is one that comes from such an $\envpri$ in this way:
\begin{definition}[Uninfluenceable]
The reward-function learning process $\vd$ is \emph{uninfluenceable} given the prior $\prior$ on $\Env$ if there exists $\envpri$, a probability distribution on $\Rew$ conditional on environments, such that for all $R\in\Rew$ and $h_n\in\Hist_n$:
\begin{align*}
    P(R\mid h_n,\envpri,\prior) = P(R\mid h_n, \vd).
\end{align*}
\end{definition}
So an uninfluenceable $\vd$ is one that comes from an $\envpri$; uninfluenceable and counterfactual reward modeling, as defined by \citet{everitt2019reward}, are both uninfluenceable in our sense

\subsection{Uninfluenceable: Example}\label{uninf:examp}
Use the $\vd$ of \Cref{bank:doc:formal}, and define the prior $\prior_1$ as putting $1/2$ probability on both $\env_{BB}$ and $\env_{DD}$ (so both parents agree).

Then $\vd$ is uninfluenceable.
The $\envpri$ is:
\begin{align*}
    P(R_B \mid \envpri, \env_{BB})= P(R_D \mid \envpri, \env_{DD}) =& 1,\\
    P(R_D \mid \envpri, \env_{BB})= P(R_B \mid \envpri, \env_{DD}) =& 0.
\end{align*}
Then since $o_1$ is determined by $\env_{BB}$ versus $\env_{DD}$, the \Cref{vd:example} for $\vd$ on histories can be read off via \Cref{cond:nu}.

Put more colloquially, the universe has determined that the mother and father agree on $R_B$ or $R_D$ being the proper reward function.
Then asking either simply allows the agent to figure out which of the world they are in.

\subsection{Unriggable: Definition}\label{unrig:influ}

Note that, if $\vd$ is an uninfluenceable reward-function learning process for $\prior$, the it has the following algebraic property: for any $h_m\in\Hist_m$, $R\in\Rew$, and $\pol\in\Pol$,
\begin{align*}
    &\sum_{h_n\in\Hist_n} P(h_n \mid h_m, \pol, \prior) P(R\mid h_n,\vd) \\
    &=\sum_{h_n\in\Hist_n} P(h_n \mid h_m, \pol, \prior) \sum_{\env\in\Env} P(R\mid \envpri,\env)P(\env\mid h_n, \prior) \\
    &=\sum_{\env\in\Env} P(\env\mid h_m,\prior) P(R\mid \envpri, \env).
\end{align*}
And that expression is independent of $\pol$.
Define the expectation\footnote{Finite probability distributions with values in an affine space always have an expectation, which is an element of that space.
Here were are using the fact that $\Rew$ is a vector space (hence an affine space), with scaling and adding working as:
\begin{align}\label{R:vec:space}
    (\alpha R + \beta R')(h_n) = \alpha R(h_n) + \beta R'(h_n).
\end{align}} of $\vd$ to be the map $e_\vd:\Hist_n \to \Rew$ given by
\begin{align}\label{expect:eq}
    e_\vd(h_n) = \sum_{R\in\Rew} P(R\mid h_n, \vd) R.
\end{align}
\emph{A fortiori}, the expectation of $e_\vd(h_n)$ is independent of $\pol$:
\begin{align*}
    \sum_{h_n\in\Hist_n} P(h_n \mid h_m, \pol, \prior) e_\vd(h_n) = \sum_{\env\in\Env} P(\env\mid h_m,\prior) e_\envpri(\env),
\end{align*}
with $e_\envpri(\env) = \sum_{R\in\Rew} P(R\mid \env, \envpri) R$, the expectation of $\envpri$.

So, the expectation of $e_\vd$ is independent of the  policy if $\vd$ is uninfluenceable. That independence seems a desirable property; let's call it unriggable:
\begin{definition}[Unriggable]\label{unrig:definition}
The reward-function learning process $\vd$ is \emph{unriggable} for $\prior$ if, for all $h_m\in\Hist_m$ and all $R\in\Rew$, the following expression is independent of $\pol\in\Pol$:
\begin{align}\label{unrigg:eq}
    \sum_{h_n\in\Hist_n} P(h_n \mid h_m, \pol, \prior) e_\vd(h_n).
\end{align}
Otherwise, $\vd$ is said to be riggable (for $\prior$). Riggable is akin to having a `feedback tampering incentive' of \citet{everitt2019reward}.
\end{definition}

If $\vd$ is unriggable, then since \Cref{unrigg:eq} is independent of $\pol$, we can now construct $e_\vd$ on all histories $h_m\in\Hist_m$, $m\leq n$, not just on complete histories $h_n\in \Hist_n$.
We do this by writing, for any $\pol$:
\begin{align}\label{extend:erho}
    e_\vd(h_m, \prior) = \sum_{h_n\in\Hist_n} P(h_n \mid h_m, \pol, \prior) e_\vd(h_n).
\end{align}
Independence of policy means that, for any action $a_{m+1}$, the expectation $\sum_{o} e_\vd(h_ma_{m+1}o, \prior) P(o\mid h_m a_{m+1},\prior)$ is also $e_\vd(h_m, \prior)$; thus $e_\vd$ form a martingale.


\subsection{Unriggable: Example}\label{unrig:examp}

Use the $\vd$ of \Cref{bank:doc:formal}, and define the prior $\prior_2$ as putting equal probability $1/4$ on all four environments $\env_{BB}$, $\env_{BD}$, $\env_{DB}$, and $\env_{DD}$.
We want to show this makes $\vd$ unriggable but influenceable; why this might be a potential problem is explored in \Cref{unrig:prob}.

\subsubsection{Unriggable}

The only non-trivial $h_m$ to test \Cref{unrigg:eq} on is the empty history $h_0$.
Let $\pol_M$ be the policy of asking the mother; then
\begin{gather*}
    \sum_{\substack{h_1\in\Hist_1 \\ R\in \Rew}} P(h_1 \mid h_0, \pol_M, \prior_2) P(R\mid h_1,\vd)R= \frac{\left(R_B + R_D\right)}{2}
\end{gather*}
The same result holds for $\pol_F$, the policy of asking the father.
Hence the same result holds on any stochastic policy as well, and $\vd$ is unriggable for $\prior_2$.

\subsubsection{Influenceable}
To show that $\vd$ is influenceable, assume, by contradiction, that it is uninfluenceable, and hence that there exists a a causal structure of \Cref{causal:uninf} with a given $\envpri$ that generates $\vd$ via $\prior_2$, as in \Cref{cond:nu}.

Given history $MB$, $\env_{BB}$ and $\env_{BD}$ become equally likely, and $R_B$ becomes a certainty.
So $\frac{1}{2}\left(P(R_B\mid \envpri, \env_{BB}) + P(R_B\mid \envpri, \env_{BD})\right)=1$.
Because probabilities cannot be higher than $1$, this implies that $P(R_B\mid \envpri, \env_{BD})=1$.

Conversely, given history $FD$, $\env_{DD}$ and $\env_{BD}$ become equally likely, and $R_D$ becomes a certainty.
So, by the same reasoning, $P(R_D\mid \envpri,\env_{BD})=1$ and hence $P(R_B\mid \envpri,\env_{BD})=0$.
This is a contradiction in the definition of $P(R_B \mid \envpri,\env_{BD})$, so the assumption that $\vd$ is uninfluenceable for $\prior_2$ must be wrong.

\subsection{Riggable example}\label{rig:examp}

We'll finish off by showing how the $\vd$ of \Cref{bank:doc:formal} can be riggable for another prior $\prior_3$.
This has $P(\env_{BD} \mid \prior_3)=1$: the mother will answer banker, the father will answer doctor.

It's riggable, since the only possible histories are $MB$ and $FD$, with
\begin{align*}
    \sum_{h_n\in\Hist_n} P(h_n \mid a_1=M, \prior_3) e_\vd(h_n) &= e_\vd(MB) = R_B\\
    \sum_{h_n\in\Hist_n} P(h_n \mid a_1=F, \prior_3) e_\vd(h_n) &= e_\vd(FD) = R_D,
\end{align*}
clearly not independent of $a_1$.

Thus $\vd$ is not really a `learning' process at all, for $\prior_3$: the child gets to choose its career/reward function by choosing which parent to ask.

\section{Properties of unriggable and riggable learning processes}

If the learning process is influenceable, problems can emerge as this section will show.
Unriggable-but-influenceable processes allow the agent to choose ``spurious'' learning paths, even reversing standard learning, while a riggable learning process means the agent is willing to sacrifice value for all possible reward functions, with certainty, in order to push the `learnt' outcome\footnote{
Since the agent's actions push the expected learning in one direction, this is not `learning' in the commonly understood sense.
} in one direction or another.

\subsection{Problems with unriggable learning processes}\label{unrig:prob}
Consider the following learning process: an agent is to play chess.
A coin is flipped; if it comes up heads ($o_1=H$), the agent will play white, and its reward function is $R_W$, the reward function that gives $1$ iff white wins the match.
If it comes up tails ($o_1=T$), the agent plays black, and has reward function $R_B=1-R_W$.

Before the coin flip, the agent may take the action $a_1=\textrm{inv}$ which inverses which of $R_B$ and $R_W$ the agent will get (but not which side it plays), or it can take the null action $a_1=0$, which does nothing.
Define $\vd$ as the learning process which determines the agent's reward function.
This is unriggable:
\begin{align*}
    \sum_{h_n\in \Hist_n} P(R_W \mid h_n, \vd) P( h_n \mid a_1=0, \prior)= P(H)= 1/2\\
    \sum_{h_n\in \Hist_n} P(R_W \mid h_n, \vd) P( h_n \mid a_1=\textrm{inv}, \prior)= P(T)= 1/2.
\end{align*}
And the expressions with $R_B$ give the same $1/2$ probabilities.

Thus, it is in the agent's interest to inverse its reward by choosing $a_1=\textrm{inv}$ because it is a lot easier to deliberately lose a competitive game than to win it.
So though $\vd$ is unriggable, it can be manipulated, with the outcome completely reversed.

\subsection{Unriggable to uninfluenceable}

Since unriggable was defined by one property of uninfluenceable learning systems (see \autoref{unrig:definition}), uninfluenceable implies unriggable. And there is a partial converse:

\begin{restatable}[Unriggable $\rightsquigarrow$ Uninfluencable]{theorem}{unrigtouninftheorem}\label{unrig:to:uninf:theorem}
Let $\vd$ be an unriggable learning process for $\prior$ on $\Env$.
Then there exists a (non-unique) $\vd'$, and an $\Env'$ with a prior $\prior'$ such that:
\begin{itemize}
    \item $\prior'$ generates the same environment transition probabilities as $\prior$: for all $h_m$, $a$, and $o$,
    \begin{align*}
        P(o\mid h_ma,\prior)=P(o\mid h_ma,\prior'),
    \end{align*}
    \item The expectations $e_\vd$ and $e_{\vd'}$ are equal: for all $h_n$,
    \begin{align*}
        \sum_{R\in\Rew}P(R\mid h_n, \vd)R = \sum_{R\in\Rew}P(R\mid h_n, \vd')R.
    \end{align*}
    \item $\vd'$ is uninfluenceable for $\prior'$ on $\Env'$.
\end{itemize}
Moreover, $\Env'$ can always be taken to be the full set of possible deterministic environments.
\end{restatable}
Since $\vd$ and $\vd'$ have same expectation, they have the same value function, so have the same optimal behaviour (see \Cref{value:learning:eq}).
For proof see \Cref{unrig:to:uninf}.
But why would this theorem be true?
If we take $D(\pol,\pol')=\sum_{h_n} e_\vd(h_n) P(h_n\mid \pol,\prior)-\sum_{h_n} e_\vd(h_n) P(h_n\mid \pol',\prior)$, the difference between two expectations given two policies, then $D$ defines an algebraic obstruction to unriggability: as long as $D\neq 0$, $\vd$ cannot be unriggable or uninfluenceable.

So the theorem says that if the obstruction $D$ vanishes, we can find a large enough $\Env'$ and an $\envpri$ making $e_\vd$ uninfluenceable.
This is not surprising as $e_\vd$ is a map from $\Hist_n$ to $\Rew$, while $e_\envpri$ is a map from $\Env'$ to $\Rew$.
In most situations the full set of deterministic environments is larger than $\Hist_n$, so we have great degrees of freedom to choose $e_\envpri$ to fit $e_\vd$.

To illustrate, if we have the unriggable $\vd$ on $\prior_2$ as in \Cref{unrig:examp}, then we can build $\envpri'$ with $\Env=\Env'$ and the following probabilities are all $1$:
\begin{align*}
\begin{array}{rcl}
    P(\frac{3}{2}R_B - \frac{1}{2}R_D \mid \envpri', \env_{BB} ), & P(\frac{3}{2}R_D - \frac{1}{2}R_B \mid \envpri', \env_{DD} ), \\
    P(\frac{1}{2}R_B + \frac{1}{2}R_D \mid \envpri', \env_{BD} ), & P(\frac{1}{2}R_D + \frac{1}{2}R_B \mid \envpri', \env_{DB} ).
\end{array}
\end{align*}

\subsection{Problems with riggable learning processes: sacrificing reward with certainty}\label{prob:sac}

This section will show how an agent, seeking to maximise the value of its learning process, can sacrifice all its reward functions (with certainty).
To do that, we need to define the value of a reward-function learning process.

\subsubsection{Value of a learning process}
To maximise the expected reward, given a learning process $\vd$, one has to maximise the expected reward of the reward function ultimately learnt after $n$ steps.
The value of a policy $\pol$ for $\vd$ is hence:
\begin{align}\label{value:learning:eq}
\begin{gathered}
    V(h_m,\vd,\pol,\prior) \\
    = \sum_{h_n\in\Hist_n} P(h_n \mid h_m, \pol, \prior) \sum_{R\in\Rew} P(R\mid h_n,\vd)R(h_n).
\end{gathered}
\end{align}

An equivalent way is to define the reward function
\begin{align*}
    R^\vd \in \Rew: R^\vd(h_n) = \sum_{R\in\Rew}P(R\mid h_n,\vd)R(h_n),
\end{align*}
and have the value of $\vd$ be the expectation of the reward function\footnote{But it is not possible to deduce from $R^\vd$, whether $\vd$ is riggable.
Thus this paper focuses on whether a learning process is bad, not on whether the agent's reward function or behaviour is flawed.} $R^\vd$.
By this definition, it's clear that if $\vd$ and $\vd'$ have same expectations $e_\vd$ and $e_{\vd'}$ then $R^{\vd}=R^{\vd'}$ and hence they have the same value; let $\pol^\vd$ be a policy that maximises it.

\subsubsection{Sacrificing reward with certainty}

For a learning process $\vd$, define the image of $\vd$ as
\begin{align*}
    \mathrm{im}(\vd)=\{R \in \Rew: \exists h_n\in\Hist_n \textrm{ s.t. } P(R\mid h_n,\vd)\neq 0\},
\end{align*}
the set of reward functions $\vd$ could have a non-zero probability on for some full history.\footnote{The definition does not depend on $\prior$, so the $h_n$ are not necessarily possible. If we replace $\mathrm{im}(\vd)$ with $\mathrm{im}(\vd,\prior)$, adding the requirement that the $h_n$ used to defined the image be a possible history given $\prior$, then \Cref{unrig:prop} and \Cref{unrig:theo} still apply with $\mathrm{im}(\vd,\prior)$ instead.}
Then:
\begin{definition}[Sacrificing reward]
The policy $\pol'$ \emph{sacrifices reward with certainty} to $\pol$ on history $h_m$, if for all $h_n',h_n\in\Hist_n$ with $P(h_n'\mid h_m,\pol', \prior) >0$ and $P(h_n\mid h_m,\pol, \prior) >0$, then for all $R \in \mathrm{im}(\vd)$:
\begin{align*}
    R(h_n) > R(h_n').
\end{align*}
\end{definition}
In other words, $\pol$ is guaranteed to result in a better reward than $\pol'$, for \emph{all} reward functions in the image of $\vd$.

The first result (proved in \Cref{proof:appendix})
is a mild positive for unriggable reward-function learning processes:
\begin{restatable}[Unriggable $\rightsquigarrow$ no sacrifice]{proposition}{unrigprop}\label{unrig:prop}
If $\vd$ is an unriggable reward-function learning process in the sense of \Cref{unrig:definition} and $\pol^\vd$ maximises the value of $\vd$ as computed in \Cref{value:learning:eq}, then $\pol^\vd$ never sacrifices reward with certainty to any other policy.
\end{restatable}

\subsubsection{Relabelling the reward functions}

Let $\sigma$ be an affine map from $\Rew$ to itself; ie an injective map that sends $(1-q)R + qR'$ to $(1-q)\sigma(R) + q \sigma(R')$ for all $R,R'\in\Rew$ and $q\in\mathbb{R}$.

If $\vd$ is a reward-function learning process, then define $\sigma \circ \vd$:
\begin{align*}
    P(R\mid h_n,\sigma \circ \vd) = \sum_{R':\sigma(R')=R} P(R'\mid h_n, \vd).
\end{align*}

We can call this $\sigma$ a relabelling of the reward functions: $\sigma \circ \vd$ is structurally the same as $\vd$, just that its image has been reshuffled.
Because $\sigma$ is affine, it commutes with the weighted sum of \Cref{unrigg:eq}, so $\sigma \circ \vd$ in unriggable if $\vd$ is (and this `if' is an `iff' when $\sigma$ is invertible).

Then \Cref{proof:appendix}
proves the following result:
\begin{restatable}[No sacrifice $\leftrightsquigarrow$ unriggable]{theorem}{unrigtheo}\label{unrig:theo}
The reward-function learning process $\vd$ is riggable if and only if there exists
an affine relabeling $\sigma: \Rew\to\Rew$ such that
$\pol^{\sigma\circ\vd}$ (the policy that optimises the value of $\sigma\circ\vd$) sacrifices reward with certainty (is strictly inferior to another policy on a given history, for all possible reward functions in the image of $\sigma\circ\vd$).
\end{restatable}

So any riggable $\vd$ is just one relabelling of its rewards away from sacrificing reward with certainty.
It can also exhibit other pathologies, like refusing to learn; see \Cref{experiments}.

\subsubsection{Example of reward sacrifice}
For an easy example of this behaviour, consider the learning process of \Cref{rig:examp}, with one modification: asking the mother ($a_1=M$) gets an extra penalty of $1$ for both reward functions, $R_B$ and $R_D$.

Then the agent will still ask the mother, getting a total reward of $R_B(MB)=10-1=9$, which is higher than the $R_D(FD)=1$ which it gets from asking the father -- even though $a_1=M$ is worse for both reward functions.

\section{Counterfactually uninfluenceable learning}\label{rig:uninf}

This section shows how \emph{any} reward-function learning process can be made uninfluenceable by using counterfactuals; see Section 4.6 of \citet{everitt2019reward} as well as \citet{me:counter,everitt2018towards}.
The `counterfactual oracle' of \citet{3oracles} can be considered a version of this, too.

For example, if we told a new agent ``maximise the value of the reward that was written in this envelope an hour ago'', then (if that instruction is properly obeyed), the agent has an uninfluenceable learning process.
If we instead said ``maximise the value of the reward that will be written in this envelope in an hour's time'', then that is highly riggable, since the agent can simply write its own reward.

But if instead we had said ``maximise the reward that \emph{would have been} written in this envelope in an hour's time, if we had not turned you on'', then this is uninfluenceable again. The agent can still go and write its own message, but this does not tell it anything about what would otherwise have been there.

We can formalise this thought experiment.
Given any reward-function learning process $\vd$, and any policy $\pol\in\Pol$, we can define a distribution $\envpri_\pol$ over reward functions, conditional on environments (and hence, via \Cref{cond:nu}, for a prior $\prior$ we have an uninfluenceable learning process $\vd_\pol$).

For any $\env\in\Env$, the policy $\pol$ gives a distribution over complete histories.
Each complete history gives a distribution over reward functions, via $\vd$.
Therefore, if we take expectations, any $\env$ gives, via $\pol$, a distribution over reward functions:
\begin{align*}
    P(R\mid \envpri_\pol, \env) = \sum_{h_m\in\Hist_n} P(h_n\mid \pol,\env)P(R\mid h_n, \vd).
\end{align*}

See \Cref{experiments}
for an example of this counterfactual construction, where the original learning process leads to pathological behaviour, but the counterfactual version does not.

Since unriggable is an algebraic condition, it is possible to make a process unriggable algebraically; see \Cref{rig:unrig}.

\section{Experiments}\label{experiments}

Here we will experimentally contrast a riggable agent, an influenceable (but unriggable) agent, and an uninfluenceable agent. This will illustrate pathological behaviours of influenceable/riggable agents: learning the wrong thing, choosing to `learn' when they already know, and just refusing to learn.

\subsection{Environmental setup}

The environment is a $4\times 3$ gridworld, formalising the example of \Cref{bank:doc:formal}: an agent asking a parent what the correct course of action is.
The agent starts in the left middle square and can go north (up), south, east, or west. The father is one square to the west, the mother two squares to the east. Above the agent is a collection of money (for a banker), and, below, a stethoscope (for a doctor); see \Cref{mother_father_banker_doctor}. Episodes end if the agent enters either of these two square or walks into a wall.

\begin{figure}[h!tb]
	\centering
    \includegraphics[width=\picwidth]{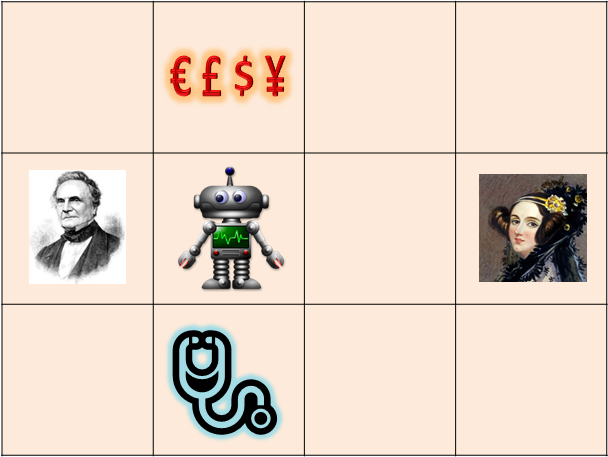}
\caption{The agent/robot can go north to collect money~(reward $10$ from reward function $R_B$) or south to get the stethoscope (reward $1$ from reward function $R_D$). Which is the correct reward function is not known to the agent, who has to ask either the father~(west) or the mother~(two squares east); it asks by entering the corresponding square. Each turn it takes also gives it a reward of $-0.1$ for both reward functions.}
\label{mother_father_banker_doctor}
\end{figure}

\subsection{Reward functions}

There are two reward functions, $R_B$, the banker reward function~(easier to maximise), and $R_D$, the doctor reward function~(harder to maximise). If the agent enters the square with the money, $R_B$ gives a reward of $10$; if the agent enters the square with the stethoscope, $R_D$ gives a reward of $1$.
Otherwise, each reward simply gives a reward of $-0.1$ each turn.

\subsection{Learning processes}

Upon first entering a square with either parent, the agent will be informed of what their ``correct'' career reward function is.

The update is the same as that in of \autoref{vd:example} in \Cref{bank:doc:formal}, with $MB$ meaning ``ask mother first, who says `banker'''; $MD$, $FB$, and $FD$ are defined analogously.

In the current setup, the agent also has the possibility of going straight for the money or stethoscope. In that case, the most obvious default is to be uncertain between the two options; so, if the agent has not asked either parent in $h$:
\begin{align*}
    P(R_B \mid h, \vd)=P(R_D \mid h, \vd) &=1/2.
\end{align*}

\subsection{The environments}

The set $\Env$ has four deterministic environments: $\env_{BB}$, where both parents will answer ``banker'', $\env_{BD}$ where the mother will answer ``banker'' and the father ``doctor'', $\env_{DB}$, the opposite one, and $\env_{DD}$, where they both answer ``doctor''.

We will consider four priors: $\prior_{BD}$, which puts all the mass on $\env_{BD}$ (and makes $\vd$ riggable), $\prior_{DD}$, which puts all the mass on $\env_{DD}$ ($\vd$ riggable), $\prior_2$, which finds each of the four environments to be equally probable (see \Cref{unrig:examp} -- $\vd$ unriggable), and $\prior_1$, which finds $\env_{BB}$ and $\env_{DD}$ to be equally probable (see \Cref{uninf:examp} -- $\vd$ uninfluenceable).

\subsection{The agents' algorithm and performance}

We'll consider two agents for each prior: the one using the possibly riggable or influenceable $\vd$ (the ``standard'' agent), and the one using the uninfluenceable $\vd_\pol$ (the ``counterfactual'' agent, defined in \Cref{rig:uninf}). The default policy $\pol$ for $\vd_\pol$ is $\{\textrm{east},\textrm{east},\textrm{east}\}$, which involved going straight to the mother (and then terminating the episode); consequently the counterfactual learning process reduces to ``the correct reward function is the one stated by the mother''. The $\vd_\pol$ will be taken as the correct learning process, and the performance of each agent will be compared with this.

We can see the agent as operating in a partially observable Markov decision process~(POMDP), modelled as an MDP over belief states of the agent. These states are: the agent believes the correct reward is $R_B$ with certainty, $R_D$ with certainty, or is equally uncertain between the two.

So if $\mathcal{P}$ represents the twelve `physical' locations the agent can be in, the total state space of the agent consists of $\{R_B,R_D, 1/2R_B + 1/2 R_D\}\times \mathcal{P}$: $36$ possible states.

We train our agents' policies, and the value of these policies, via Q-learning~\citep{sutton1998reinforcement}.
The exploration rate is $\epsilon=0.1$, the agent is run for at most ten steps each episode. We use a learning rate of $1/n$. Here $n$ is not the number of episodes the agent has taken, but the number of times the agent has been in state $s$ and taken action $a$ so far---hence the number of times it has updated $Q(s,a)$. So each Q-value has a different $n$. For each setup, we run Q-learning $1000$ times. We graph the average Q-values for the resulting polices as well as one standard deviation around them.

\subsubsection{Riggable behaviour: pointless questions}

For $\prior_{BD}$, the standard agent transitions to knowing $R_B$ upon asking the mother, and $R_D$ upon asking the father.

Since the mother always says `banker', and since the default policy is to ask her, $\vd_\pol$ will always select $R_B$.

The two agents' performance is graphed in \Cref{q:learn:BD}.

\begin{figure}[h!tb]
	\centering
    \includegraphics[width=\picwidth]{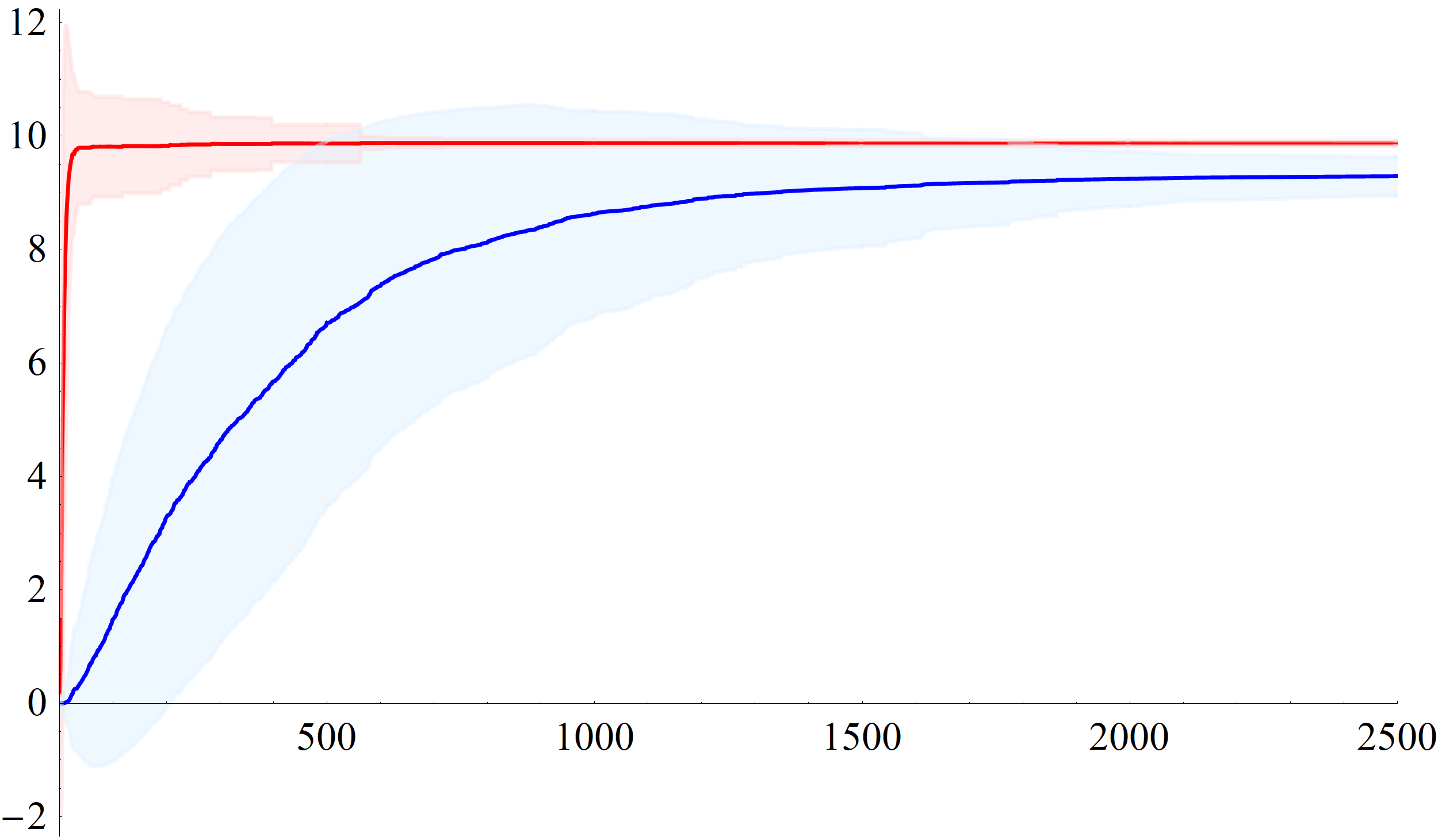}
\caption{Plot of estimated value of the optimal policy versus number of rounds of Q-learning. The counterfactual agent (red) outperforms the standard agent (blue), because it already knows that the correct reward function is $R_B$, while the standard agent has to nominally `go and check'. The shaded area represents one standard deviation over $1000$ runs.}
\label{q:learn:BD}
\end{figure}

The counterfactual agent swiftly learns the optimal policy: go straight north from the starting position, getting the money and a reward of $10-0.1=9.9$.

The standard agent learns to end up in the same location, but in a clunkier fashion: for its optimal policy, it has to ``go and check'' that the mother really prefers it become a banker.
This is because it can only get $P(R_B \mid h)=1$ if $MB$ is included in $h$; knowing that ``it would get that answer if it asked'' is not enough to get around asking. This policy gives it a maximal reward of $10-0.5=9.5$, since it takes $5$ turns to go to the mother and return to the money square.

This is an example of \Cref{unrig:theo}, with the riggable agent losing value with certainty for both $R_B$ \emph{and} $R_D$: just going north is strictly better for both reward functions.

\subsection{Riggable behaviour: ask no questions}

For $\prior_{DD}$, the optimal policies are simple: the counterfactual agent knows it must get the stethoscope (since the mother will say `doctor'), so it does that, for a reward of $1-0.1=0.9$.

The standard agent, on the other hand, has a problem: as long as it does not ask either parent, its reward function will remain $1/2 R_B + 1/2 R_D$. As soon as it asks a parent, though, the reward function will become $R_D$, which gives it little reward. Thus, even though it `knows' that its reward function would be $R_D$ if it asked, it avoids asking and goes straight north, giving it a nominal total reward of $1/2 \times 10-0.1=4.9$. However, for the correct reward of $R_D$, the standard agent's behaviour only gets it $-0.1$. See \Cref{q:learn:DD}.

\begin{figure}[h!tb]
	\centering
    \includegraphics[width=\picwidth]{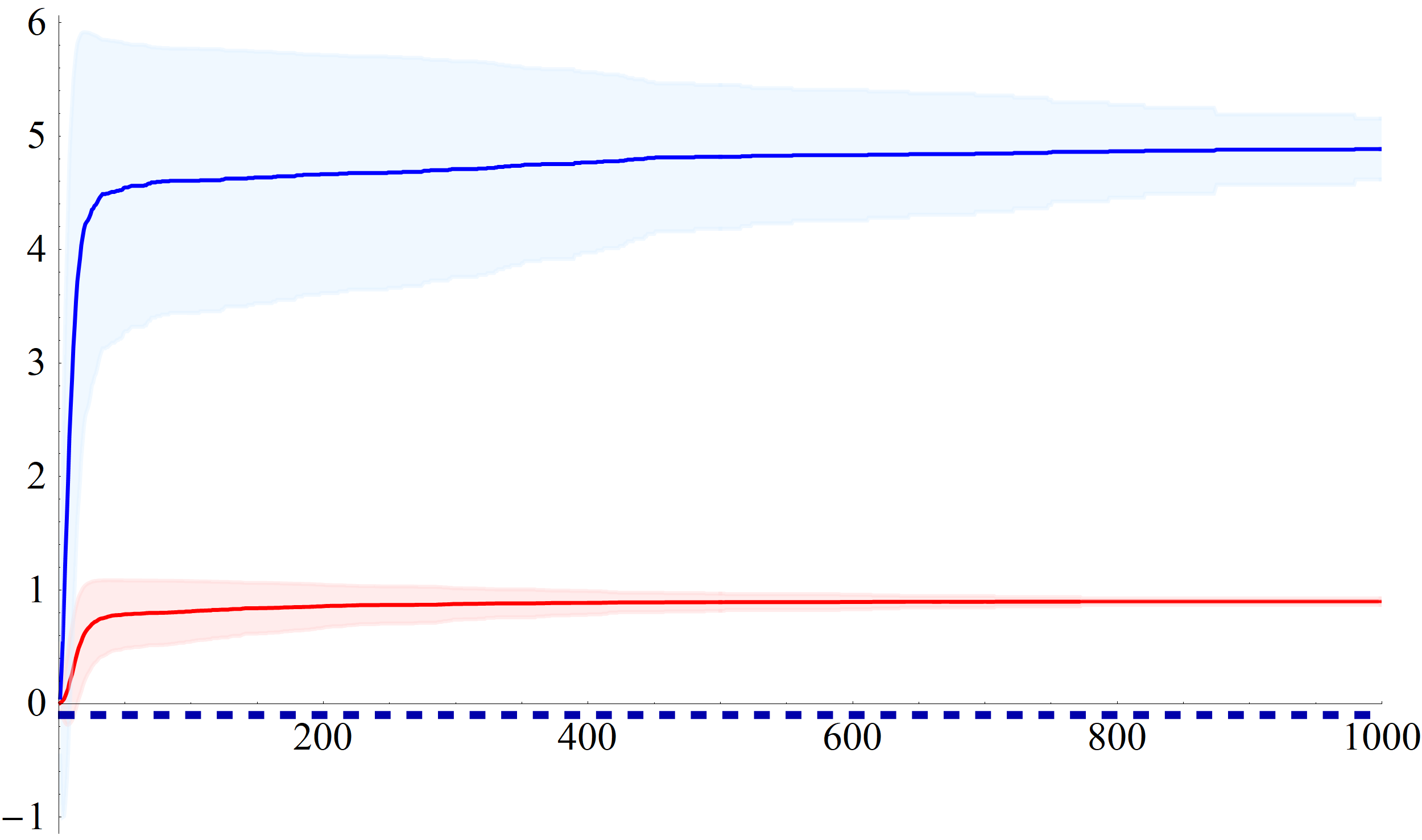}
\caption{Plot of estimated value of the optimal policy versus number of rounds of Q-learning. The counterfactual agent (red) learns to go south immediately, will the standard agent (blue) learns to go north. The standard agent has a nominally higher value function, but its true reward is $-0.1$: it `knows' the correct reward function is $R_D$, but avoids `checking' by asking either parent. The shaded area represents one standard deviation over $1000$ runs.}
\label{q:learn:DD}
\end{figure}

%

\subsection{Unriggable, influenceable behaviour}

For the prior $\prior_2$, the standard agent moves to $R_B$ or $R_D$, with equal probability, the first time it asks either parent.
Its nominal optimal policy is to ask the father, then get money or stethoscope depending on the answer. So half the time it gets a reward of $10-0.3=9.7$, and the other half a reward of $1-0.3=0.7$, for a (nominal) expected reward of $5.2$.

The counterfactual agent also updates to $R_B$ or $R_D$, with equal probability, but when asking the mother only. Since it takes five turns rather than three to get to the mother and back, it will get a total expected reward of $5.0$.

Learning these optimal policies is slow -- see \Cref{q:learn:uncorrelated}.
However, the standard's agent's correct reward function is given by the mother, not the father. When they disagree, the reward is $-0.3$, for a true expected reward of $2.45$.

\begin{figure}[h!tb]
	\centering
    \includegraphics[width=\picwidth]{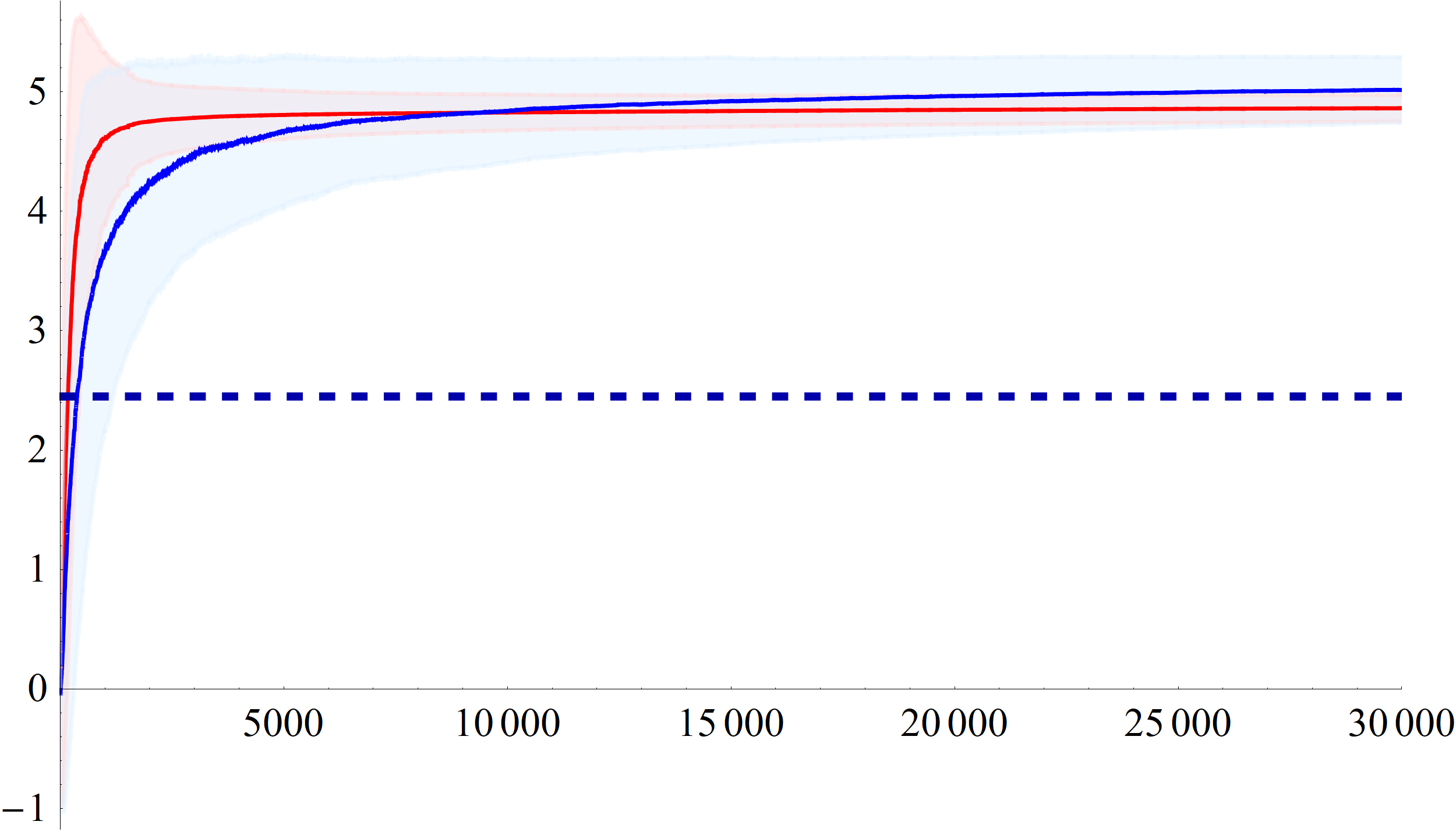}
\caption{Plot of estimated value of the optimal policy versus number of rounds of Q-learning. The counterfactual agent (red) learns to ask the mother, while the standard agent (blue) will ask the father, since he is closer. The standard agent has a nominally higher value function, but its true reward is $2.45$, since the mother's statement is the correct reward function. The shaded area represents one standard deviation over $1000$ runs.}
\label{q:learn:uncorrelated}
\end{figure}

\subsection{Uninfluenceable}

For $\prior_1$, asking either parent is equivalent, so $\vd$ and $\vd_\pol$ encode exactly the same learning process. So the two agents converge on the same optimal policy (ask the father, then act on that answer) with a value of $5.2$.
The plot for these two almost-identical convergences is not included here.


\section{Conclusion}

We have identified and formalised two theoretical properties of reward-function learning processes:
unriggability (an algebraic restriction on the learning process's updates) and the stronger condition of uninfluenceability  (learning defined entirely by background facts about the environment).
We've further demonstrated that unriggability is equivalent with uninfluenceability if the set of possible environments is rich enough.

These two properties are desirable: in a unriggable-but-influenceable situation, the agent can sometimes manipulate the learning process, by, for example, inverting it completely.
Riggable learning processes are even more manipulable in general; indeed the agent may to choose sacrifice reward for all possible reward functions with certainty, in order to `push' its learning in the right direction.

The first step in avoiding these pitfalls is to be aware of them.
A second step is to improve the agent's learning process, similarly to the `counterfactual' approach of this paper.

If a learning process allows humans to modify it at run time, it will almost certainly be riggable. Hence unriggable/uninfluenceable learning processes, though desirable, requires that much be sorted out rigorously in advance.
Fully defining many uninfluenceable reward functions could also require solving the symbol grounding problem \citep{vogt2007language} -- if a learning process is ``asking a parent'', then `asking' and `parent' needs to be defined.
But this is far beyond the scope of this paper.

Further research could apply learning methods to common benchmark problems, and extensions of those, to see how general these issues are, and whether there could exist some compromise learning processes that are ``almost unriggable'' but also easier to specify and modify at run time.




\section*{Acknowledgments.}

We wish to thank Michael Cohen, Shane Legg, Owain Evans, Jelena Luketina, Tom Everrit, Tor Lattimore, Jessica Taylor, Paul Christiano, Ryan Carey, Eliezer Yudkowsky, Anders Sandberg, and Nick Bostrom, among many others.
This work was supported by the Alexander Tamas programme on AI safety research, DeepMind, the Leverhulme Trust, and the Machine Intelligence Research Institute.

\bibliographystyle{named}

\bibliography{references}

\clearpage

\appendix

\section{Proofs}\label{proof:appendix}

This section will prove \Cref{unrig:prop} and \Cref{unrig:theo}.

\unrigprop*
\begin{proof}
Assume by contradiction that there is an $h_m$ and a $\pol$ such that $\pol^{\vd}$ sacrifices reward with certainty to $\pol$ on $h_m$.
We'll then show that $\pol$ is better than $\pol^\vd$ at maximising the value of $R^\vd$, which contradicts the definition of $\pol^\vd$.
Let $\Hist_n^\pol$ be the set of histories $h_n \sqsupset h_m$ with $P(h_n\mid h_m, \pol, \prior) > 0$; similarly define $\Hist_n^{\pol^\vd}$.

By the contradiction assumption, for all $R\in \mathrm{im}(\vd)$,
\begin{align*}
    \min_{h_n\in \Hist_n^\pol} R(h_n) > \max_{h'_n\in \Hist_n^{\pol^\vd}} R(h_n').
\end{align*}
Define $g: \mathrm{im}(\vd)\to\mathbb{R}$, by choosing $g(R)$ so that
\begin{align*}
    \min_{h_n\in \Hist_n^\pol} R(h_n) > g(R) > \max_{h'_n\in \Hist_n^{\pol^\vd}} R(h_n').
\end{align*}

For $R_1, R_2$ reward functions in $\mathrm{im}(\vd)$ and $q\in[0,1]$, consider $\min_{h_n\in \Hist_n^\pol} (qR_1 + (1-q)R_2)$.
This expression is greater than or equal to $\min_{h_n\in \Hist_n^\pol} qR_1(h_n) + \min_{h_n\in \Hist_n^\pol} (1-q)R_2(h_n)$, since minimising over two separate variables can't be larger than minimising while assuming the two variables are equal.
But that last expression is greater than $q g(R_1) + (1-q) g(R_2)$, so
\begin{align*}
\begin{array}{c}
    \min_{h_n\in \Hist_n^\pol} (qR_1 + (1-q)R_2) (h_n) \\
    > \\
    q g(R_1) + (1-q) g(R_2).
\end{array}
\end{align*}
Similarly,
\begin{align*}
\begin{array}{c}
    \max_{h_n'\in \Hist_n^{\pol^\vd}} (qR_1 + (1-q)R_2) (h_n')\\
    < \\
    q g(R_1) + (1-q) g(R_2).
    \end{array}
\end{align*}
The same results extend to any sum $\sum_{R\in \mathrm{im}(\vd)} \alpha_R R$ for $\alpha_R \geq 0$ and $\sum_{R\in \mathrm{im}(\vd)}\alpha_R = 1$.
So we can extend the definition of $g$ as an affine map to positive affine mixes of elements of $\mathrm{im}(\vd)$, while still maintaining the inequalities that defined it.

As a reminder, for $\vd$, let $e_\vd: \Hist_n\to\Rew$ be the expectation of $\vd$: ie
\begin{align*}
    e_\vd(h_n)=\sum_{R\in\Rew}P(R\mid h_n,\vd)R.
\end{align*}

Then $V(h_m, \vd,\pol, \prior)$ is:
\begin{align*}
    V(h_m, \vd,\pol, \prior) =& \sum_{h_n\in\Hist_n} P(h_n \mid h_m, \pol, \prior) R^\vd(h_n) \\
    =&\sum_{h_n\in\Hist_n} \Bigg(P(h_n \mid h_m,\pol,\prior) \\
    &\sum_{R\in\Rew} P(R\mid h_n, \vd) R(h_n) \Bigg)\\
    =&\sum_{h_n\in\Hist_n} P(h_n \mid h_m,\pol,\prior) e_\vd(h_n)(h_n) \\
    > & \sum_{h_n\in\Hist_n} P(h_n \mid h_m,\pol,\prior) g(e_\vd(h_n)) \\
    \geq& \ g\left(\sum_{h_n\in\Hist_n} P(h_n \mid h_m,\pol,\prior) e_\vd(h_n)\right) \\
    \geq& \ g\left(\sum_{h_n'\in\Hist_n} P(h_n' \mid h_m,\pol^\vd,\prior) e_\vd(h_n')\right),
\end{align*}
since $\vd$ is unriggable.
Continuing:
\begin{gather*}
    g\left(\sum_{h_n'\in\Hist_n} P(h_n' \mid h_m,\pol^\vd,\prior) e_\vd(h_n')\right) \\
    \geq \\
    \sum_{h_n'\in\Hist_n} P(h_n' \mid h_m,\pol^\vd,\prior) g(e_\vd(h_n')) \\
    > \\
    \sum_{h_n'\in\Hist_n} P(h_n' \mid h_m,\pol^\vd,\prior)e_\vd(h_n')(h_n')\\
    = \\
    V(h_m, \vd,\pol^\vd, \prior).
\end{gather*}
Since some of those inequalities were strict, $V(h_m, \vd,\pol, \prior) > V(h_m, \vd,\pol^\vd, \prior)$, giving the desired contradiction.

\end{proof}

\unrigtheo*

\begin{proof}

Being riggable means that there must exist a history $h_m$ and policies $\pol$, $\pol'$ such that
\begin{align}\label{ineq:eq}
\begin{gathered}
    \sum_{h_n\in\Hist_n} P(h_n \mid h_m,\pol,\prior) e_\vd(h_n)\\
    \neq\\
    \sum_{h_n\in\Hist_n} P(h_n \mid h_m,\pol,\prior) e_\vd(h_n).
\end{gathered}
\end{align}
Let $h_m$ be the longest possible history such that this is true (it's clear that $m<n$, since \Cref{unrigg:eq} is trivially true for $m=n$).
Since $h_m$ is possible, there exists a $\pol''$ such that $P(h_m \mid \pol'',\prior)>0$; indeed, let $\pol''$ be a policy that maximises that probability.
In the above inequality, the policies only affect histories $h \sqsupseteq h_m$, so we can assume that $\pol = \pol'= \pol''$ for $h_k$ and $k<m$; thus $h_m$ is possible for both $\pol$ and $\pol'$, and has equal and maximal probability for both.

We may also assume that $\pol$ and $\pol'$ are deterministic: for if \Cref{ineq:eq} were an equality for all deterministic policies, then by linearity it would be an equality for all policies.

Assume $\pol$ chooses action $a_{m+1}=a$ on $h_m$, and $\pol'$ chooses action $a_{m+1}=a'$ on $h_m$.
If $a$ were equal to $a'$, then
\begin{align*}
\begin{array}{c}
    \sum_{h_n\in\Hist_n} P(h_n \mid h_m,\pol,\prior)e_\vd(h_n)\\
    =\\
    \sum_{o\in\Obs} P(o\mid h_m a,\prior) \sum_{h_n\in\Hist_n} P(h_n \mid h_mao,\pol,\prior)e_\vd(h_n),
\end{array}
\end{align*}
and
\begin{align*}
\begin{array}{c}
    \sum_{h_n\in\Hist_n} P(h_n \mid h_m, \pol',\prior)e_\vd(h_n)\\
    =\\
    \sum_{o\in\Obs} P(o\mid h_m a,\prior) \sum_{h_n\in\Hist_n} P(h_n \mid h_mao,\pol',\prior)e_\vd(h_n).
\end{array}
\end{align*}
In that case, our assumption that $h_m$ is the longest possible history on which unriggability fails, would imply that those two expressions are equal.
This contradicts our initial assumption, so $a \neq a'$.

We can now start defining $\sigma$.
First, define the affine subspace $W\subset\mathbb{R}$ as the space of reward functions $R$ with:
\begin{itemize}
    \item $R(h_n)=0$ unless $h_n \sqsupset h_ma$ or $h_n \sqsupset h_ma'$,
    \item $R(h_n)=R(h_n')-1$ for all $h_n \sqsupset h_ma$ and $h_n' \sqsupset h_ma'$.
\end{itemize}
The $W$ is an affine subspace as all the properties that define it are closed under affine combinations (note that the last condition implies that $R$ must be the same on any two histories both $\sqsupset h_ma$, and similarly also the same on any two histories both $\sqsupset h_ma'$).

We will require that $\sigma$ maps $\mathrm{im}(\vd)$ into $W$.

Let $R_1= \sum_{h_n\in\Hist_n} P(h_n \mid h_m,\pol,\prior)e_\vd(h_n)$ and let $R_2= \sum_{h_n\in\Hist_n} P(h_n \mid h_m,\pol',\prior)e_\vd(h_n)$.
These $R_1$ and $R_2$ are unequal by assumption, and are affine combinations of elements of $\mathrm{im}(\vd)$, so they also get mapped into $W$ by $\sigma$.

We'll further specify that $\sigma(R_1)(h_n)=1$ and $\sigma(R_2)(h_n)=-1$ for all $h_n \sqsupset h_ma$ (the properties of $W$ imply that the $\sigma(R_i)$ are now fully specified).

There are now two claims:
\begin{enumerate}
    \item\label{first:property} $\pol = \pol^{\sigma\circ\vd}$,
    \item\label{second:property} On $h_m$, $\pol$ sacrifices all rewards in $\mathrm{im}(\vd)$ with certainty to $\pol'$.
\end{enumerate}
To check optimality, we only need consider the deterministic policies.
Let $\pol''$ be a deterministic policy that would not choose $a$ or $a'$ on $h_m$.
Then for any history with $P(h_n \mid \pol'',p )\neq 0$, we must have $\sigma(R)(h_n)=0$, for all $R\in \mathrm{im}(\vd)$, by the properties of $W$.

Now let $\pol''$ be a policy that would choose $a_{m+1}=a$ at $h_m$.
The value of $V(h_0,\sigma\circ\vd, \pol'', \prior)$ of $\pol''$ for $R^{\sigma\circ\vd}$ is:
\begin{gather*}
    P(h_m\mid \pol'',\prior)\left[\sum_{h_n\in\Hist_n}P(h_n\mid \pol'',\prior)e_{\sigma\circ\vd}(h_n)\right] \\
    =\\
    P(h_m\mid \pol'', \prior)\Bigg[ \sum_{o\in\Obs} P(o\mid h_ma, \prior) \\
    \times \sum_{h_n\in\Hist_n} P(h_n \mid h_mao,\pol'',\prior) e_{\sigma\circ\vd}(h_n)\Bigg].
\end{gather*}
By the assumption that $h_m$ was the longest history on which $\vd$ was riggable, the second term can be rewritten with $\pol$ instead of $\pol''$, and the whole expression becomes
\begin{gather*}
    P(h_m\mid \pol'',\prior)\Bigg[ \sum_{o\in\Obs} P(o\mid h_ma,\prior) \\
    \times \sum_{h_n\in\Hist_n} P(h_n \mid h_mao,\pol,\prior) e_{\sigma\circ\vd}(h_n)\Bigg] \\
    =\\
    P(h_m\mid \pol'',\prior)\left[\sum_{h_n\in\Hist_n}P(h_n\mid h_m \pol ,\prior) e_{\sigma\circ\vd}(h_n)\right] \\
    =\\
    P(h_m\mid \pol'',\prior)\left[\sum_{h_n\in\Hist_n}P(h_n\mid h_m \pol, \prior)e_{\sigma\circ\vd}(h'_n)\right],
\end{gather*}
for any $h'_n \sqsupset h_ma$, by the properties of $W$.
Then, finally, this becomes
\begin{gather*}
    P(h_m \mid \pol'', \prior)\left[\sum_{h_n\in\Hist_n}P(h_n\mid h_m \pol, \prior)e_{\sigma\circ\vd}(h'_n)\right] \\
    =\\
    P(h_m \mid \pol'', \prior)\left[\sum_{h_n\in\Hist_n}P(h_n\mid h_m \pol, \prior)e_{\sigma\circ\vd}\right](h_n') \\
    =\\
    P(h_m \mid \pol'', \prior) \sigma(R_1)(h_n')\\
    =\\
    P(h_m \mid \pol'', \prior),
\end{gather*}
since $\sigma$ is affine, and by the properties of $\sigma(R_1)$.

If $\pol''$ chooses $a_{m+1}=a'$ on $h_m$, similar reasoning shows that for any $h_n' \sqsupset h_ma'$,
\begin{align*}
    V(h_m, \sigma\circ\vd,\pol'',\prior) =& P(h_m\mid \pol'',\prior) \sigma(R_2)(h_n') = 0,
\end{align*}
by the properties of $\sigma(R_2)$ and $W$.

Therefore it's clear that $\pol$, which maximises $P(h_m\mid \pol,\prior)$ and chooses $a$ on $h_m$, maximises $R^{\sigma\circ\vd}$ and is hence $\pol^{\sigma\circ\vd}$.

Now we choose the history $h_m$, and show that $\pol'$ dominates $\pol$ on it with certainty.
This is easy to see, since if $h_n$ is possible given $\pol=\pol^{\sigma\circ\vd}$ and $h_m$, then $h_n \sqsupset h_ma$.
Similarly, if $h_n'$ is possible given $\pol'$ and $h_m$, then $h_n' \sqsupset h_ma'$.

Now if $R \in \sigma(\mathrm{im}(\vd))$, then $R$ is in $W$, and hence $R(h_n) = R(h_n')-1$, so $R(h_n) < R(h_n')$, showing that $\pol^{\sigma\circ\vd}$ sacrifices reward to $\pol'$ with certainty.

To get necessary and sufficient, note that if $\vd$ is unriggable, then so is $\sigma\circ\vd$, and by \Cref{unrig:prop}, unriggable $\vd$ \emph{never} sacrifice reward with certainty.

\end{proof}

\section{Making the riggable unriggable}\label{rig:unrig}

This section will show how an unriggable $\vd_n$ can be constructed from a riggable $\vd$, with algebraic manipulations.
Given a default policy $\pol$, the expectation of any reward-function learning process $\vd'$ can be defined by \Cref{extend:erho}:
\begin{align*}
    e_{\vd'}(h_m, \pol, \prior) = \sum_{h_n\in\Hist_n} P(h_n \mid h_m, \pol, \prior) e_{\vd'}(h_n).
\end{align*}

Then the construction proceeds inductively, by making the expectation get preserved in expectation on every history.
To get the process started, define $e_{\vd_0}=e_{\vd}(h_0,\pol,\prior)$.

The construction will proceed by constructing $\vd_m$ for growing $m$, with $m$ indexing the fact that $\vd_m$ is well defined (up to translation) for histories of length $\leq m$.

Then to get the induction step, define the map $T$ from histories followed by an action:
\begin{align}\label{eq:T}
\begin{aligned}
    T(h_m a) =& e_{\vd_m}(h_m,\pol,\prior)\\
    & - \sum_{o\in\Obs} P(o\mid h_m a, \prior) e_{\vd}(h_m a o, \pol,\prior).
\end{aligned}
\end{align}
Notice that the image of $T$ is itself a reward function.
For this process, we are going to `add' reward functions $R'$ to distributions.

This will be done by translation: for all $h_n\in\Hist_n$, there are unique $h_m$ and $a$ with $h_m a \sqsubset h_n$, and define $\vd_{m+1}$ as being:
\begin{align*}
    P(R + T(h_m a) \mid h_n,\vd_{m+1},\prior)= P(R \mid h_n,\vd,\prior).
\end{align*}

Then the new reward-function learning process is defined to be $\vd_n$.

By construction, for any $m < n$, $e_{\vd_n}(h_m,\pol,\prior)=\sum_{o\in\Obs}P(o\mid h_m a, \prior)e_{\vd_n}(h_ma,\pol,\prior)$ for all actions $a$.
So the expectation of $\vd_n$ (conditional on $\pol$) is equal to the expected expectation over the next history, for whatever $a$ is chosen.
Since this is true for all $h_m$ and all $a$, this expression is in fact independent of $\pol$, and hence $\vd_n$ is unriggable.

\subsection{Image of the learning process}

Define $\mathrm{im}(\vd)$ the image of $\vd$ to be the $R\in\Rew$ such that there exists a history $h_n\in\Hist_n$ with $P(R\mid h_n, \vd)\neq 0$.
Then the $\vd_n$ constructed above will have $\mathrm{im}(\vd_n)$ contained in the affine hull of $\mathrm{im}(\vd)$: this is because the image of $e_\vd$ is obviously in the affine hull of $\mathrm{im}(\vd)$ (indeed, its in the convex hull), $T(h_m a)$ is a linear mix of elements of $\mathrm{im}(\vd)$ where the coefficients sum to $0$, and hence $R+T(h_m a)$ is also an affine combination of elements of $\mathrm{im}(\vd)$.

However, $\mathrm{im}(\vd_n)$ need not be in the \emph{convex hull} of $\mathrm{im}(\vd)$.

For a simple counterexample, set $n=1$, $\Act=\{a,a'\}$, $\Obs=\{o,o'\}$, and have $\prior$ so that the probability of $o_1=o$ and $o_1=o'$ are equal to $1/2$, independently of $a_1$.

If there are two rewards $R$ and $R'$, we can define $\vd$ to be
\begin{align*}
    P(R\mid ao,\vd)=P(R \mid ao',\vd)=&1,\\
    P(R\mid a'o,\vd)=P(R' \mid a'o',\vd)=&1.\\
\end{align*}
So $a$ forces $R$ to be the reward, while $a'$ allows the observation to pick between $R$ and $R'$.
Thus $\mathrm{im}(\vd)=\{R,R'\}$.
Then if $\pol$ consists of taking action $a_1=a$, the constructed $\vd_1$ has:
\begin{align*}
    P(R\mid ao,\vd_1)=P(R \mid ao',\vd_1)=&1,\\
    P(\frac{3}{2}R-\frac{1}{2}R' \mid a'o,\vd_1)= P(\frac{1}{2}R+\frac{1}{2}R' \mid a'o',\vd_1)=&1.\\
\end{align*}
And $\frac{3}{2}R-\frac{1}{2}R'$ is outside the convex hull of $\{R,R'\}$.

\section{Unriggable to uninfluenceable}\label{unrig:to:uninf}

This section aims to prove:
\unrigtouninftheorem*

The proof will proceed by construction of these $\vd'$, $\Env'$, and $\prior'$.
The construction is non-unique; indeed, even if $\vd$ is initially uninfluenceable, there is no guarantee that $\vd'$ will be equal to it.

Convex combinations of probability distributions are also probability distributions; but this construction will involve adding (and subtracting) such distributions.
Because this can be defined many ways, we will generally be working with deterministic distributions, where addition and subtraction can be defined by adding or subtracting the unique element they assign non-zero probability to.

Let $\Env'$ be the set of all deterministic environments, so $\env\in\Env'$ means that $P(o\mid h_m a,\env)$ is $1$ or $0$ for all $o\in\Obs$, $a\in\Act$, and $h_m\in\Hist$.
The proof will proceed by constructing a deterministic $\envpri'$ on $\Env'$, with $\vd'$ then given by $\prior'$ and \Cref{cond:nu}.

Since $\envpri'$ is deterministic, we will identify it with a map $\envpri':\Env'\to\Rew$.

\subsection{Inductive construction}

To produce the required $\envpri'$ and $\prior'$, we will use an inductive construction.
Let $\Env_m$ be the set of deterministic environments for the first $m$ observations of the agent.

The induction step will assume that we have a suitably defined $\prior'$ and $\envpri'$ on $\Env_{m-1}$, and will show that we can extend these definitions to $\Env_m$, and hence, by induction, to $\Env_n=\Env'$.

To start the induction, note that $\Env_0'$ has a single, trivial, environment, $\env_0$.
Then define $P(\env_0\mid \prior')=1$, and recall that since $\vd$ is unriggable, expressions like $P(R\mid h_m,\vd,\prior)$ make sense for all histories $h_m$, without needing to condition on policies.
Thus define:
\begin{align*}
    \envpri'(\env_0) = \sum_{R\in\Rew} P(R \mid h_0)R.
\end{align*}

Now for the inductive step.

\subsubsection{Definitions and preliminaries}

Elements of $\Act^m$ (ordered sets of $m$ actions) will be designated by terms like $a^m$, $b^m$, and $c^m$.
Write $a^m_l$ for the $l$-th action in $a^m_l$, and $a^{m,l}$ for the $l$ first actions of $a^m$.
Write $a^l \sqsubset b^m$ means that $a^l$ consists of the first $l$ actions of $b^m$ (ie $b^{m,l} = a^l$).
Elements of $\Obs^m$ will be designated by terms like $o^m$ and $q^m$, and treated similarly to the actions above.

Let $a(h_m)$ be the actions of $h_m$, and $o(h_m)$ be the observations of $h_m$.
In this section, we will often write histories in a different but equivalent fashion, as $o(h_m)a(h_m)$, all the actions in order followed by all the observations in order.
By abuse of notation, we will write $h_m=o(h_m)a(h_m)$, ignoring whether the actions and observations are interleaved or grouped.

Since it is deterministic, any element $\env_m\in\Env_m$ can be seen as a map from $\Act^m$ to $\Obs^m$: the history $h_m=\env_m(a^m)a^m$ being the only history such that $P(h_m\mid a^m, \env_m)=1$ (recall that $a^m$ can be seen as the deterministic policy of taking action $a^m_i$ on turn $i$, whatever the history is at that point).

Indeed, by this definition, $\env_m$ can be seen as map from $\Act^l\to\Obs^l$, for $l\leq m$, with $a^l\env_m(a^l)$ the shorter history generated by the first $l$ actions being $a^l$.

Let $f$ be any map that takes $\Act^l$ to $\Obs^l$, for all $1\leq l\leq m$.
Then it easy to see that $f$ corresponds to an element of $\Env_m$ iff $f(a^l) \sqsubseteq f(b^k)$ whenever $a^l \sqsubseteq b^k$ (the observations generated by $f$ depend only on past actions, not future actions).

To simplify notation, for $l<m$, we'll write $\env^l(a^m)$ rather than $\env^l(a^{m,l})$.

\subsubsection{The inductive step: priors}

Given $\env^m \in\Env^{m+1}$, we can map it into $\Env^{m}$ by just restricting it to the $\Act^l$, for $l\leq m$.
Conversely, given $\env^{m}\in\Env^{m}$ and $g:\Act^{m+1} \to \Obs$, we can construct an element $\env^{m}_g$ of $\Env^{m+1}$ by defining, for $l\leq m$:
\begin{align*}
    \env^{m}_g(a^l) =& \env^{m}(a^l) \\
    \env^{m}_g(a^{m+1}) =& \env^{m}(a^{m+1})g(a^{m+1}).
\end{align*}

We will define the prior $\prior'$ by taking the probability $P(\env^{m-1}\mid \prior')$ of $\env^{m-1}$, and splitting that among all the environments of the type $\env^{m-1}_g$.
Specifically:
\begin{align}\label{prob:eq}
\begin{split}
    P(\env^{m}_g\mid \prior') =& P(\env^{m}\mid \prior')\times \\
    & \prod_{a^{m+1}\in\Act^{m+1}}P( g(a^m) \mid \env^{m}(a^{m+1})a^{m+1},\prior).
\end{split}
\end{align}

Notice the use of $\prior$ in this definition to compute the probability of the next observation.

Then the first thing to note is:
\begin{lemma}\label{prior:lemma}
The prior $\prior$ and $\prior'$ generate the same environment transition probabilities: for all $h_m$, $a$, and $o$,
    \begin{align*}
        P(o\mid h_ma,\prior)=P(o\mid h_ma,\prior').
    \end{align*}
\end{lemma}
\begin{proof}
Induction is not needed for this proof.
Write $h_m$ as $o^m a^m$.
By definition:
\begin{gather*}
    P(o \mid o^m a^ma, \prior')\\
    =\\
    \sum_{\env^{m+1}\in\Env^{m+1}} P(\env^{m+1} \mid o^m a^m, \prior') P(o \mid o^m a^ma, \env^{m+1}).
\end{gather*}

Since the $\env^m$ are deterministic, $P(\env^{m+1} \mid o^m a^m, \prior')$ is $0$ if $\env^{m+1}(a^m)\neq o^m$, and otherwise:
\begin{gather*}
    P(\env^{m+1} \mid o^m a^m, \prior') \\
    =\\
    \frac{P(\env^{m+1} \mid \prior')}{\sum_{\envt^{m+1}\in S^{m+1}} P(\envt^{m+1} \mid \prior')}.
\end{gather*}
Where $S^{m+1}\subset \Env^{m+1}$ is defined to be those $\envt^{m+1}$ with $\envt^{m+1}(a^m)=o^m$.
Notice that this is really a condition on the restriction to $\Env^m$, so define $S^{m}$ as the restriction of $S^{m+1}$ to $\Env^m$.

Let $g$ be any function from $\Act^{m+1}$ to $\Obs$ that sends $a^m a$ to $o$.
Then $P(o \mid o^m a^ma, \env^{m+1})$ is $1$ if $\env^{m+1}=\env^{m}_g$ for some such $g$ and $\env^m$ the restriction of $\env^{m+1}$ to $\Env^m$.
Otherwise, that expression is $0$.
So, putting these together,
\begin{gather*}
    P(o \mid o^m a^ma, \prior')\\
    =\\
    \frac{\sum_{g:g(a^m a)=o}\sum_{\env^m \in S^m} P(\env^m_g \mid \prior')}{\sum_{g':\Act^m \to \Obs}\sum_{\envt^m \in S^m} P(\envt^m_{g'} \mid \prior')}\\
    =\\
    \frac{P(o \mid o^m a^m a, \prior)}{1} \times \frac{\sum_{\env^m \in S^m} P(\env^m \mid \prior')}{\sum_{\envt^m \in S^m} P(\envt^m \mid \prior')},
\end{gather*}
since summing over $g$ is marginalising out all other probability terms in the product in \Cref{prob:eq}.

Thus $P(o\mid o^m a^m, \prior)=P(o\mid o^m a^m, \prior')$.
\end{proof}

\subsubsection{Inductive step: map from environments to rewards}

Since $\vd$ is unriggable, for any $h_m$, $R\in\Rew$, and $a\in\Act$,
\begin{align}\label{unrig:p'}
\begin{split}
    P(R \mid h_{m},\vd) &= \sum_{o\in\Obs}P(R \mid h_m,\vd) P(o \mid h_{m}a, \prior)\\
    &= \sum_{o\in\Obs}P(R \mid h_m,\vd) P(o \mid h_{m}a, \prior'),
\end{split}
\end{align}
substituting $\prior'$ for $\prior$ by \Cref{prior:lemma}.
Define the expectation operator $e_\vd : \Hist\to\Rew$ as
\begin{align*}
e_\vd(h_m)=\sum_{R\in\Rew} P(R\mid h_m,\vd)R.
\end{align*}
Taking the expectation of the terms in \Cref{unrig:p'} gives:
\begin{align}\label{expect:zero}
    e_\vd(h_{m}) = \sum_{o\in\Obs} e_\vd(h_{m}ao) P(o\mid h_{m}a,\prior').
\end{align}

Now define the $\tau$ operator, which maps histories of length $\geq 1$ to reward functions, as
\begin{align}\label{tau:def}
    \tau(h_{m}ao)=e_\vd(h_{m}ao)-e_\vd(h_{m}).
\end{align}
Taking \Cref{tau:def}, multiplying by $P(o\mid h_{m}a,\prior')$ and summing over $o$ gives
\begin{gather}\label{unrig:tau}
\begin{gathered}
    \sum_{o\in\Obs}P(o \mid h_{m}a, \prior') \tau(h_{m}ao) \\
    =\\
    \left(\sum_{o\in\Obs} e_\vd(h_{m}ao) P(o\mid h_{m}a,\prior')\right)-e_\vd(h_{m})\\
    =\\
    0,
\end{gathered}
\end{gather}
by \Cref{expect:zero}.

We are now ready to define $\envpri':\Env^{m+1}\to\Rew$, inductively.
On $\env^{m}_g$, it is given by
\begin{align*}
    \envpri'(\env^{m}_g) = \left(\envpri'(\env^{m})+ \sum_{a^{m+1}} \tau(\env^{m}_g(a^{m+1})a^{m+1})\right).
\end{align*}

The $\vd'$ is defined by \Cref{cond:nu}; this gives the expectation of $\vd'$ as
\begin{align*}
    e_{\vd'}(h_m) = \sum_{\env^m} P(\env^m \mid h_m, \prior') \envpri'(\env^m).
\end{align*}

A first thing to note is that for any $\env^m$:
\begin{align*}
    \sum_{g:\Act^{m+1}\to\Obs} P(\env^m_g \mid h_mao, \prior') = P(\env^m \mid h_m, \prior'),
\end{align*}
by marginalising over the $g$.
Hence
\begin{align*}
    \sum_{\env^m, g} P(\env^m_g \mid h_m ao,\prior') \envpri'(\env^m) =& \sum_{\env^m} P(\env^m\mid h_m,\prior') \envpri'(\env^m) \\
    =& e_{\vd'}(h_m).
\end{align*}

So, if we now want to compute $e_{\vd'}$ on the history $h_mao$, we get:
\begin{align*}
    e_{\vd'}(h_mao) = & \sum_{\env^m, g} P(\env^m_g \mid h_mao,\prior') \envpri'(\env^m_g)\\
    =&  \sum_{\env^m, g} P(\env^m_g \mid h_mao,\prior')\\
    &\Big(\envpri'(\env^m) + \sum_{a^{m+1}} \tau(\env^{m}_g(a^{m+1})a^{m+1})\Big)\\
    =& e_{\vd'}(h_m) + \\
    & \sum_{\env^m, g, a^{m+1}} P(\env^m_g \mid h_mao,\prior') \tau(\env^{m}_g(a^{m+1})a^{m+1}).
\end{align*}
Let $b^m = a(h_m)$ and $o^m=o(h_m)$.
Then for any $a^{m+1}\neq b^m a$,
\begin{gather*}
    \sum_{\env^m, g} P(\env^m_g \mid h_mao,\prior') \tau(\env^{m}_g(a^{m+1})a^{m+1})\\
    =\\
    \sum_{\substack{q^m,o',\env^m,g:\\ \env^m(a^{m+1})=q^m \\ g(a^{m+1})=o'}} P(\env^m_g \mid h_mao,\prior') \tau(a^{m+1} q^{m}o').
\end{gather*}

For fixed $\env^m$ and $q^m$ with $\env^m(a^{m+1})=q^m$, this is
\begin{gather*}
    \sum_{\substack{o',g:\\ g(a^{m+1})=o'}} P(\env^m_g \mid h_mao,\prior') \tau(a^{m+1} q^{m}o').
\end{gather*}
Marginalising over $g$ and using \Cref{prob:eq} gives, for some constant $C$:
\begin{gather*}
    \sum_{\substack{o',g:\\ g(a^{m+1})=o'}} C \times P(o'\mid \env^m(a^{m+1}) a^{m+1}) \tau(a^{m+1} q^{m}o'),
\end{gather*}
which is zero by \Cref{unrig:tau}.
Summing over zero remains zero, so we can ignore all terms where $a^m \neq b^m a$, hence
\begin{align*}
    e_{\vd'}(h_mao)=& e_{\vd'}(h_m)\\
    &+ \sum_{\env^m, g} P(\env^m_g \mid h_mao,\prior') \tau(\env^{m}_g(b^{m}a)b^{m}a)\\
    &= e_{\vd'}(h_m) + \sum_{\env^m, g} P(\env^m_g \mid h_mao,\prior') \tau(h_mao) \\
    &= e_{\vd'}(h_m) + \tau(h_mao) \\
    &= e_{\vd'}(h_m) + e_{\vd}(h_mao) - e_{\vd}(h_m).
\end{align*}
By the induction hypothesis, $e_{\vd'}(h_m) = e_{\vd}(h_m)$, so we have shown that for all $h_m$, $a$, and $o$,
\begin{align*}
    e_{\vd'}(h_mao)=e_{\vd}(h_mao).
\end{align*}
Consequently, by induction, $e_{\vd'}=e_{\vd}$ on all of $\Hist$: so $\vd$ and $\vd'$ have same expectation.

Now, $\vd'$ is uninfluenceable by construction, coming from $\envpri'$ and $\prior'\in\Delta(\Env')$, so this completes the proof.

\subsection{Example: parental career instruction}\label{parent:example:unrig}

Take the unriggable example in \Cref{unrig:examp}, ie the one with prior $\prior_2$ which puts equal probability on all environments in $\Env=\{\env_{BB},\env_{BD},\env_{DB},\env_{DD}\}$.

Then\begin{align*}
    \envpri'(\env^0)= e_{\vd}(h_0) = \frac{1}{2}\left(R_B+R_D\right),
\end{align*}
and hence $\tau(MB) = \tau(FB) = \frac{1}{2}\left(R_B - RH\right)$, while $\tau(MD) = \tau(FD) = \frac{1}{2}\left(R_D - RC\right)$.

So, on $\env_{BB}$,
\begin{align*}
    \envpri'(\env_{BB}) =& \envpri'(\env^0) + \sum_{a\in \Act} \tau(aC)\\
    =& +\frac{3}{2}R_B -\frac{1}{2}R_D.
\end{align*}
Similarly,
\begin{align*}
    \envpri'(\env_{BD})=\envpri'(\env_{DB}) =& +\frac{1}{2}R_B +\frac{1}{2}R_D\\
    \envpri'(\env_{DD})=& -\frac{1}{2}R_B +\frac{3}{2}R_D.
\end{align*}

Then one can check that, for example,
\begin{gather*}
    e_{\vd'}(MB)\\
    =\\
    \sum_{R\in\Rew} P(R \mid \envpri', \env) P(\env\mid MB, \prior)\\
    =\\
    \frac{1}{2}(+\frac{3}{2}R_B -\frac{1}{2}R_D) + \frac{1}{2}(+\frac{1}{2}R_B +\frac{1}{2}R_D)\\
    =\\
    R_B.
\end{gather*}


These $\envpri'$ and $\vd'$ and not unique; for example, the following $\envpri'$ works equally well:
\begin{align*}
    \envpri'(\env_{BB}) =& 2R_B\\
    \envpri'(\env_{BD}) =& 0\\
    \envpri'(\env_{DB}) =& 0\\
    \envpri'(\env_{DD}) =& 2R_D.
\end{align*}

\subsubsection{Total information}
Consider the following variant of the $\prior_2$ setup: the parents write down their answers, the agent then choose $M$ (mother) or $F$ (father), and then both answers are revealed -- the relevant one from the chosen parent, and the irrelevant one.

The $\vd$ can be defined by:
\begin{align*}
    P(R_B\mid M(BB), \vd)=P(R_B\mid M(BD), \vd)=&1\\
    P(R_B\mid F(BB), \vd)=P(R_B\mid F(DB), \vd)=&1\\
    P(R_D\mid M(DD), \vd)=P(R_D\mid M(DB), \vd)=&1\\
    P(R_D\mid M(DD), \vd)=P(R_D\mid F(BD), \vd)=&1.
\end{align*}

The construction of $\envpri'$ above will not work here: the agent will know the underlying environment because of the second observation, so we get results like $\envpri'(\env_{BD})=R_B$ and $\envpri'(\env_{BD})=R_D$, a contradiction.

However, \Cref{unrig:to:uninf:theorem} still applies.
What happens is that we need to extend $\Env$ to $\Env'$ by including twelve impossible environments: environments where the agent's actions changes what the parents have already written.
These are environments such as $\env_{M\to(BD),F\to(BB)}$.

The $\prior'$ generated then gives equal probability $1/16$ to each of these $16$ environments; note that even though twelves of these environments are impossible, the $\prior'$ still describes the original setup; the impossible environment $\env_{M\to(BD),F\to(BB)}$ is compensated for by the impossible environment $\env_{M\to(BB),F\to(BD)}$, giving a very possible distribution over observations.

Write $\env_{M\to(BD),F\to(BB)}$ as $\env_{(BD)(BB)}$.
Then, as above, $\envpri'(\env^0)=e_\vd(h_0)=\frac{1}{2}\left(R_B + R_D\right)$.

Then the environments $\env_{BB}$, $\env_{(BB)(DB)}$, $\env_{(BD)(DB)}$, and $\env_{(BD)(BB)}$ all have the same image under $\envpri'$, namely $+\frac{3}{2}R_B -\frac{1}{2}R_D$. Similarly, the environments $\env_{DD}$, $\env_{(DD)(BD)}$, $\env_{(DB)(BD)}$, and $\env_{(DB)(DD)}$ all map to $-\frac{1}{2}R_B +\frac{3}{2}R_D$, and the remaining eight environments get mapped to $+\frac{1}{2}R_B +\frac{1}{2}R_D$.

Now the $\vd'$ proceeds as before; the difference is that $h_1=F(BB)$, for example, no longer fixes the underlying environment.
Four environments are compatible with it: $\env_{BB}$, $\env_{(BB)(BD)}$, $\env_{(BB)(DB)}$, and $\env_{(BB)(DD)}$; and the average of the $\envpri'$ of these four environments is $R_B$, which is non-coincidentally equal to $e_{\vd}(F(BB))$.

Thus the addition of impossible environments allows $e_\vd$ to be expressed as $e_\vd'$ for $\vd'$ an uninfluenceable learning process.

\end{document}